\documentclass{article}

\usepackage[preprint]{neurips_2024}
\usepackage{amsmath, amssymb}
\usepackage{amsthm}
\usepackage{geometry}
\usepackage{hyperref}
\usepackage{algorithm}
\usepackage[noend]{algorithmic}
\usepackage{booktabs}
\usepackage{amsthm}
\usepackage{xcolor}
\usepackage{graphicx}
\usepackage{pifont}
\usepackage{comment}
\usepackage{enumitem}
\usepackage{wrapfig}

\newtheorem{theorem}{Theorem}
\newtheorem{lemma}{Lemma}
\newtheorem{assumption}{Assumption}
\newtheorem{remark}{Remark}

\newtheorem{innercustomthm}{Theorem}
\newenvironment{customthm}[1]{\renewcommand\theinnercustomthm{#1}\innercustomthm}{\endinnercustomthm}
\newtheorem{innercustomlem}{Lemma}
\newenvironment{customlem}[1]{\renewcommand\theinnercustomlem{#1}\innercustomlem}{\endinnercustomlem}
\newtheorem{innercustomasp}{Assumption}
\newenvironment{customasp}[1]{\renewcommand\theinnercustomasp{#1}\innercustomasp}{\endinnercustomasp}

\usepackage[utf8]{inputenc} 
\usepackage[T1]{fontenc}    
\usepackage{hyperref}       
\usepackage{url}            
\usepackage{booktabs}       
\usepackage{amsfonts}       
\usepackage{nicefrac}       
\usepackage{microtype}      
\usepackage{xcolor}         

\title{Ensuring Safety in an Uncertain Environment: Constrained MDPs via Stochastic Thresholds}

\author{%
Qian Zuo\\
School of Informatics\\
University of Edinburgh\\
\texttt{Q.Zuo-1@sms.ed.ac.uk}
\And
Fengxiang He\\
School of Informatics\\
University of Edinburgh\\
\texttt{F.He@ed.ac.uk}
}

\begin{document}

\maketitle

\begin{abstract}
  This paper studies constrained Markov decision processes (CMDPs) with constraints against stochastic thresholds, aiming at safety of reinforcement learning in unknown and uncertain environments. We leverage a Growing-Window estimator sampling from interactions with the uncertain environment to estimate the thresholds, based on which we design \textit{{S}tochastic {P}essimistic-{O}ptimistic {T}hresholding} (\textit{SPOT}), a novel model-based primal-dual algorithm for multiple constraints against stochastic thresholds. SPOT enables reinforcement learning under both pessimistic and optimistic threshold settings. We prove that our algorithm achieves sublinear regret and constraint violation, i.e., a reward regret of $\tilde{\mathcal{O}}(\sqrt{T})$ while allowing an $\tilde{\mathcal{O}}(\sqrt{T})$ constraint violation over $T$ episodes. The theoretical guarantees show that our algorithm achieves performance comparable to that of an approach relying on fixed and clear thresholds. To the best of our knowledge, SPOT is the first reinforcement learning algorithm that realises theoretical guaranteed performance in an uncertain environment where even thresholds are unknown.
\end{abstract}

\section{Introduction}
Reinforcement learning (RL) has enabled major advances in AI, such as AlphaGo \cite{silver2016mastering}, autonomous vehicle control \cite{kendall2019learning}, and LLM alignment \cite{ouyang2022training}, where an agent is trained to maximise its expected cumulative reward from the environment \cite{Sutton1998}.
Markov decision processes (MDPs) are widely adopted in reinforcement learning, where transitions and rewards depend
only on the current state rather than on the entire history \cite{bellman1957markovian}. In many applications, ensuring safety is of profound importance -- for instance, autonomous vehicles are expected to avoid collisions and adhere to traffic regulations \cite{wen2020safe}. 
This motivates constrained Markov decision processes (CMDPs), introducing safety constraints against thresholds into optimisation of the cumulative reward function \cite{altman1999constrained}. 

A substantial literature studies unknown constraints in safe RL.  Classical CMDP methods, built on linear programming and primal dual updates, achieve (near-)optimal learning with regret guarantees under model uncertainty \cite{zheng2020constrained,efroni2020exploration,wei2018online,qiu2020upper}. To reduce reliance on fully specified constraints, inverse constrained RL (ICRL) infers latent constraints or safety surrogates from demonstrations and can generalize across thresholds \cite{malik2021inverse,liu2024comprehensive}; in parallel, constraint-conditioned approaches explicitly condition policies on a threshold input to enable cross-threshold generalization \cite{yao2023constraint}. Safety-first exploration takes a different route: it does not estimate the threshold value themselves but enforces safety by learning certified safe sets \cite{wachi2020safe} or by using uncertainty to remain below a  specified safety budget \cite{wachi2023safe}. Complementary lines manage budgets and risks through state augmentation \cite{sootla2022saute}, chance constraints with randomized policies \cite{shen2024flipping} and offline adaptation schemes that switch among pretrained policies to satisfy deployment-time cost thresholds \cite{chemingui2025constraint}.

Despite this progress, most safe-RL methods still hinge on exogenous and fixed safety thresholds, rely on high-quality demonstrations, or substitute true limits with conservative surrogates \cite{wachi2024survey}. These choices fail when admissible limits are numerically unknown and fluctuate with noise or operating conditions \cite{yao2023constraint}. Our vision is to move beyond `adaptation around a given threshold' and instead learn the threshold itself as a latent, data-driven quantity from interaction and update it online with calibrated uncertainty. This motivates the following question: 
\begin{quote}
\centering
    \textit{Can we design an online CMDP 
    algorithm that achieves 
    decent theoretical guarantees
    in an environment with unknown safety thresholds?}
\end{quote}
To address this problem, this paper proposes \textit{stochastic pessimistic-optimistic thresholding} (\textit{SPOT}), an online algorithm with theoretical guarantees on both safety and performance when thresholds evolve over time. We consider CMDPs with finite horizon, similar to an array of existing works \cite{bai2020provably,bura2022dope,efroni2020exploration,ghosh2024towards,muller2024truly,stradi2024optimal,zheng2020constrained}. We make \textbf{contributions} as follows:

\begin{itemize}[leftmargin=12pt]
    \item \textbf{Mathematically grounding stochastic thresholds in CMDPs.} This paper first realises and defines a safe reinforcement learning problem in which knowledge of the environment is highly limited that even the thresholds are unknown, which is termed as CMDP-ST (CMDP with stochastic thresholds). We formalise this uncertainty by modelling the threshold as a stochastic process: instead of observing the true episodic safety thresholds ($\alpha_i$), the agent only receives noisy and per-step signals ($\tilde{\alpha}_{i,h}$), where the latent threshold is defined as their expected sum. This grounds the problem of learning to satisfy constraints online without any prior knowledge of the threshold's underlying distribution. 
    \item \textbf{The first approach to safe reinforcement learning agents in such uncertain environment.} SPOT leverages a Growing-Window estimator that, at each step and episode, uses the value of the most visited state-action pair to characterize the global threshold estimate. Based on this estimator,
    SPOT employs a 
    primal-dual scheme to integrate the thresholds into the learnt agent.  SPOT is compatible to two variants that (1) employs pessimistic thresholds to strictly enforce safety, and (2) applies optimistic thresholds to encourage exploration. By systematically navigating the latent `margin' between safety and performance without relying on any prior threshold information, our approach offers a dynamic balance between exploitation and exploration.
    \item \textbf{Theoretical guarantees on the efficiency and safety of SPOT.} We first prove that our threshold estimator is asymptotically consistent with the hidden thresholds with a guaranteed convergence speed (Section \ref{sec:estimations}). We then establish the feasibility and strong duality of this constrained optimisation problem (Section \ref{sec:Strong duality}). We prove that SPOT achieves $\tilde{\mathcal{O}}(\sqrt{T})$ regrets, for both cumulative reward and constraint violations over $T$ episodes, 
    which are comparable to its counterpart with pre-defined thresholds (Section \ref{sec:thms}). Our analysis also suggests a trade-off between efficiency and safety: a constraint that enforces stringent safety can effectively mitigate violations but incur higher reward regret, and vice versa. 
\end{itemize}

\textbf{Related Works.} Research on MDPs for online learning has seen a long history, which primarily focused on achieving a sublinear regret in stochastic or adversarial environments \cite{orabona2019modern, auer2008near, even2009online, neu2010online, rosenberg2019online, rosenberg2019online2, jin2020learning}. Building on these, researchers have integrated safety and resource constraints into MDPs, yielding CMDP methods that achieve strong sublinear regret bounds under assumptions of `full-information feedback', where the complete loss function is observable \cite{wei2018online, qiu2020upper}. Subsequent works on bandit feedback, where the agent only observes the outcomes of its chosen actions, extended these guarantees to both known- and unknown-transition models, achieving sublinear bounds on both weak reward regret and constraint violation \cite{zheng2020constrained, efroni2020exploration, bai2020provably,bura2022dope}. Recent works have pushed the frontier further, obtaining sublinear strong regret guarantees \cite{stradi2024optimal, muller2024truly,ghosh2024towards}. Efforts have also been seen in extending the results to more complex non-stationary environments (i.e., under assumptions that rewards and constraints have bounded variations), maintaining sublinear regrets even as the environment evolves \cite{ding2023provably, wei2023provably}. Although partial feedback models have been explored beyond simple bandit scenarios \cite{bacchiocchi2024markov,lin2024information}, these methods generally yield weaker theoretical guarantees. In addition, adversarial scenarios remain particularly challenging, with most existing results achieving only weak regret bounds \cite{stradi2024online,levy2023efficient}.
As discussed, most studies assume that safety thresholds can be clearly pre-defined. To our best knowledge, this paper presents the first CMDP algorithm to move beyond this limitation, learning safety thresholds online from interaction while achieving sublinear regret bounds.

\section{Preliminaries}

\paragraph{Constrained Markov decision process (CMDP)} An MDP characterises the sequential decision-making process of an agent, positioned in a state $s$ from a finite state $\mathcal{S}$ (with $|\mathcal{S}|=S$), which could select an action $a$ from a finite action space $\mathcal{A}$ (with $|\mathcal{A}|=A$) and then transition to another state $s' \in \mathcal{S}$ according to a transition function $p_h: \mathcal{S}\times\mathcal{A}\times\mathcal{S}\to [0,1]$. The course of each episode starts at an initial state $s_1\in \mathcal{S}$ and continues for $H \in \mathbb{N}$ steps, which is named as horizon. At each step $h\in [H]$, the agent receives a reward $\tilde{r}_h: \mathcal{S}\times \mathcal{A}\to [0,1]$ drawn from a distribution $\mathcal{R}_h$ with expectation reward $r_h(s,a)$, which means its underlying distribution may vary over time. The agent's trajectory is determined by a policy $\pi$ consisting of decision rules $(\pi_1,\ldots,\pi_H)\in\Pi$, where each $\pi_{h}(\cdot|s)\in \Delta(\mathcal{A})$ is a probability distribution over the action space $\mathcal{A}$. Here, $\Delta(\mathcal{A})$ denotes the probability simplex over the action space $\mathcal{A}$. A Constrained MDP (CMDP) extends an MDP by incorporating $m$ constraints, each defined by a function $g_{i,h}: \mathcal{S}\times \mathcal{A}\to [0,1]^m$ associated with a threshold. These thresholds collectively form a vector $\alpha \in [0,H]^m$. In the stationary scenario, the immediate constraint $\tilde{g}_{i,h}(s,a)\in [0,1]$ is drawn from a distribution $\mathcal{G}_{i,h}$ with expectation $\mathbb{E}_{\mathcal{G}_i}[\tilde{g}_{i,h}(s,a)]=g_{i,h}(s,a)$ for every constraint $i \in [m]$. Overall, a CMDP can be represented by the tuple $\mathcal{M}=(\mathcal{S},\mathcal{A},H,p,r,g,\alpha)$.

\paragraph{Value functions and objective functions} 
Given a policy $\pi \in \Pi$ and a vector $v \in [0,1]^{\mathcal{S}\times \mathcal{A}}$ indexed by state-action pairs, the value functions
\begin{gather*}
    V_{v,h}^{\pi} (s_h) = \mathbb{E}_{\pi, p} \left[ \sum_{h' = h}^{H} v_{h'}(s_{h'}, a_{h'}) \mid s_h \right],\text{      }
Q_{v,h}^{\pi} (s_h,a_h) = \mathbb{E}_{\pi, p} \left[ \sum_{h' = h}^{H} v_{h'}(s_{h'}, a_{h'}) \mid s_h , a_h \right],
\end{gather*}
capture the expected sum of $v$ from step $h$ onward. For brevity, let $V_v^{\pi}$ denote $V_{v,1}^{\pi} (s_1)$, the expected return starting at the initial state. The primary objective is to find a policy that maximises the expected cumulative reward while satisfying all the constraints; formally,
\begin{equation}
\label{objective}
\max_{\pi \in \Pi} \; V_r^\pi 
\quad
\text{s.t.} 
\quad
V_{g_i}^\pi \geq \alpha_i \quad (\forall i \in [m]),
\end{equation}
where $V_r^\pi$ is the expected cumulative reward value, and $V_{g_i}^\pi$ is the expected cumulative constraint value for constraint $i\in [m]$, constrained by the safety threshold $\alpha_i$. The Lagrangian of this optimisation problem is $\mathcal{L}(\pi,\lambda)=V_r^\pi+\sum_{i=1}^m \lambda_i(V_{g_i}^\pi-\alpha_i)$, where $\lambda \in \mathbb{R}_{\geq 0}^m$ is the vector of Lagrange multipliers. Consequently, problem \eqref{objective} can be cast as a saddle-point problem: $\max_{\pi \in \Pi} \;\min_{\lambda \in \mathbb{R}^m_{\ge 0}} \mathcal{L}(\pi, \lambda)$.

\paragraph{Training process} The algorithm interacts over a fixed number of $T\in \mathbb{N}$ episodes. Before each episode \(t \in [T]\), the algorithm chooses a policy \(\pi_t \in \Pi\) and executes it for one run of the horizon.
The goal is to simultaneously minimise its reward regret and constraint violation,
\begin{align*}
    \mathcal{R}_T(r) &:= \sum_{t=1}^{T} \Big( V_r^{\pi^\star} - V_r^{\pi_t}\Big), \text{      }
    \mathcal{R}_T(g) := \max_{i \in [m]} \sum_{t=1}^{T} \Big( \alpha_i-V_{g_i}^{\pi_t}\Big).
\end{align*}
The reward regret measures the cumulative gap between 
the rewards obtained by 
the selected policies and the optimal policy $\pi^\star \in \Pi$, and the constraint violation measures how much the constraints are violated during the learning process. 

\paragraph{Other notations} Throughout this paper, $\tilde{O}$ indicates asymptotic bounds to polylogarithmic factors. $\lesssim$ is used to represent $\le$ up to numerical constants or polylogarithmic factors.

\section{Uncertainty of Environment: A Stochastic Threshold Perspective}

Environments can be unknown and uncertain, including the safety thresholds. This means that an agent has to rely on incremental observations and feedback from the environment to `infer' the thresholds over time, which is later served for training itself. This section first discusses what this means to a reinforcement learning agent, then mathematically grounds it into an optimisation problem. We also discuss the properties of this optimisation problem.

\subsection{Unknown and Stochastic Environment in Terms of Thresholds}
\label{sec: New Setting}
In many real-world applications, safety or performance thresholds are not pre-defined but must be inferred from a stream of noisy observations. A representative example is personalized medicine: a treatment plan for a patient might span several weeks and total cumulative dose of a drug must stay below a toxicity level for the entire period. At each step, for instance on a daily basis, a doctor receives noisy signals from the patient's biomarkers. This feedback provides a per-step hint about the safe dosage for that day \cite{liu2024learning}. Therefore, the agent's core challenge is to aggregate these per-step, unreliable signals to form a coherent strategy for the overall episodic constraint.

\paragraph{Stochastic thresholds} Motivated by this example, we propose training a CMDP in which the safety thresholds are realized through a stochastic process. We assume a stationary environment in which, for each constraint $i\in [m]$ and step $h\in[H]$, there exists an unknown threshold $\alpha_{i,h}\in[0,1]$. When the agent visits any state-action pair $(s,a)$, it receives a bounded noisy observation $\tilde{\alpha}_{i,h}(s,a)\in[0,1]$ which is drawn from a distribution $\mathcal{L}_{i,h}$ with expectation $\mathbb{E}_{\mathcal{L}_{i,h}}[\tilde{\alpha}_{i,h}(s,a)]=\alpha_{i,h}$. The episodic threshold for constraint $i$ is $\alpha_i:=\sum_{h=1}^H\alpha_{i,h}$. Similarly, for each step $h$ and constraint $i$, the reward $\tilde{r}_{h}(s,a)$ and cost $\tilde{g}_{i,h}(s,a)$ are drawn from unknown families of distributions $\mathcal{R}_h$ and $\mathcal{G}_{i,h}$, respectively.

The fundamental difficulty is that the agent has no prior knowledge of the distributions $\{\mathcal{L}_{i,h}\}_{i,h}$, $\{\mathcal{R}_h\}_h$, or $\{\mathcal{G}_{i,h}\}_{i,h}$. Consequently, the true parameters which involve the expected reward vector $r$, constraints $g_i$ and thresholds $\alpha_i$ are all unknown. This information must be acquired online by sampling from these distributions. Algorithm \ref{alg:interaction} depicts this interaction, where at each episode $t$, the agent executes a policy $\pi_t$ and observes stochastic realisations (i.e., the samples $\tilde{r}_h$, $\tilde{g}_{i,h}$ and $\tilde{\alpha_{i,h}}$, which are then used to learn the unknown parameters.

\begin{algorithm}[t]
\caption{Agent-Environment Interaction for $t \in [T]$}
\label{alg:interaction}
\begin{algorithmic}[1]
\REQUIRE Policy $\pi_t\in \Pi$
\STATE Environment initialises state $s_1 \in \mathcal{S}$
\FOR{$h = 1, \dots, H$}
    \STATE Agent takes action $a_h \sim \pi_t(\cdot \mid s_h)$
    \STATE Agent observes reward $\tilde{r}_h^t(s_h, a_h)$, constraint cost $\tilde{g}_{i,h}^t(s_h, a_h)$, and threshold $\tilde{\alpha}_{i,h}^t(s_h,a_h)$ for all $i \in [m]$
    \STATE Environment evolves to $s_{h+1} \sim p(\cdot \mid s_h, a_h)$
\ENDFOR
\end{algorithmic}
\end{algorithm}

\textbf{Objective Function under Stochastic Thresholds} The learning objective is to find a policy $\pi^{\star}$ that maximizes the expected cumulative reward while respecting the (unknown) safety constraints:
\begin{equation*}
\max_{\pi \in \Pi} V_r^\pi \quad \text{s.t.} \quad V_{g_i}^\pi \ge \alpha_i=\sum_{h=1}^H \alpha_{i,h}, \quad \forall i \in [m],
\end{equation*}
where $V_r^{\pi}:=\mathbb{E}_{\pi,p}\left[\sum_{h=1}^H r_h\right]$ and $V_{g_i}^{\pi}:=\mathbb{E}_{\pi,p}\left[\sum_{h=1}^H g_{i,h}\right]$ are the reward and constraint expectations under the trajectory induced by $(\pi,p)$, and $\alpha_i=\sum_{h=1}^H \alpha_{i,h}$ is the threshold expectation with $\alpha_{i,h}=\mathbb{E}_{\mathcal{L}_{i,h}}[\tilde{\alpha}_{i,h}]$. 

Given a policy $\pi$ and transition function $p$, the occupancy measure $q_h^\pi$ is defined as the probability of the agent visiting state $s$ and taking action $a$ at step $h$, that is, $q_h^\pi(s,a;p)=\mathbb{P}\left\{s_h=s,a_h=a|p,\pi\right\}$ \cite{Puterman1994,altman1999constrained}.
The entire collection of these probabilities is denoted by $q:=[q_h^\pi(s,a;p)]_{H\times \mathcal{S}\times \mathcal{A}}$. Define the reward vector $r$ with entries $r_h(s,a)$ and the constraint matrix $G$  with rows $(g_{i,h}(s,a))_{h,s,a}$. Then, $V_r^{\pi}=r^{\top}q$ and $(V_{g_i}^{\pi})_{i\in[m]}= G^\top q$. Under the occupancy measure formulation, the optimisation problem is a saddle-point problem defined over the (unknown) expected reward vector $r$, constraint matrix $G$ and threshold vector $\alpha$:
\begin{equation*}\label{eq:stochastic}
    \max_{q \in \Delta{(\mathcal{M})}} \;\min_{\lambda \in \mathbb{R}^m_{\ge 0}} \mathcal{L}(q, \lambda):=r^{\top}q+\lambda^{\top}(G^{\top}q-\alpha).
\end{equation*}
The corresponding target feasible set is $\mathcal{F}_q:=\{q\in\Delta{(\mathcal{M})}|\sum_{h,s,a}q_h^\pi(s,a;p)g_{i,h}(s,a)\geq \alpha_i, \forall i\in [m]\}$,
which forms a convex polytope within the space of occupancy measures $\Delta(\mathcal{M})$.

\subsection{Estimators of Thresholds}
\label{sec:estimations}
We ask how to use the minimum amount of data while ensuring good performance in the stationary environment. To control computational cost without sacrificing statistical guarantees, we estimate thresholds using only the most recent episodes through a Growing-Window (GW) approach and then estimates at a data-efficient representative pair for each step.

Let $\{W_t\}_{t=1}^T$ be a window-length sequence that increases with episode $t$. For any constraint $i\in[m]$, step $h\in[H]$ and state-action pair $(s,a)\in\mathcal{S}\times\mathcal{A}$, the GW estimator at episode $t$ is
\begin{equation*}\label{empirical-gw}
\hat{\alpha}_{i,h}^{(W_t),t}(s,a) := \frac{\sum_{l=\max\{1,t-W_t+1\}}^t \tilde{\alpha}_{i,h}^l(s,a)\,\mathbf{1}_{\{s_h^l=s,a_h^l=a\}}}{\max\{1, N_h^{(W_t),t}(s,a)\}},
\end{equation*}
where $N_h^{(W_t),t-1}(s,a):=\sum_{l=\max\{1,t-W_t+1\}}^{t-1}\mathbf{1}_{\{s_h^l=s,a_h^l=a\}}$ counts visits to state-action pair $(s,a)$ in the last $W_t$ episodes before episode $t$. As $t$ increases, the window size $W_t$ grows. We prove its asymptotic consistency as follows.
\begin{theorem}[Asymptotic consistency of Growing-Window estimator to thresholds]\label{thm:estimator2}
    Given a confidence parameter $\delta \in (0,1)$, with probability at least $1 - \delta$, the following holds for every constraint $i \in [m]$, step $h\in [H]$, episode $t \in [T]$, and state-action pair $(s,a) \in \mathcal{S} \times \mathcal{A}$:
\begin{equation*}
    \left|\hat{\alpha}_{i,h}^{(W_t),t}(s,a)-\alpha_{i,h}\right|\leq\zeta_h^{(W_t),t}(s,a),
\end{equation*}
where $\zeta_h^{(W_t),t}(s,a)=\min\left\{1,\sqrt{\frac{4\ln(mSAHT/\delta)}{\max\{1,N_h^{(W_t),t-1}(s,a)\}}}\right\}$.
\end{theorem}

For brevity we simplify the notation $\hat{\alpha}_{i,h}^{(W_t),t}$ to $\hat{\alpha}_{i,h}^{t}$, $\zeta_h^{(W_t),t}$ to $\zeta_h^{t}$. The proof is given in Appendix \ref{app:estimation}.

\begin{remark}
    Fixed-window approaches may repeatedly discard valuable historical data, which hurts long-run convergence. Full-history methods may respond too slowly to changes. A Growing-Window approach preserves enough past information to ensure eventual convergence but can adapt to evolving conditions more rapidly.
\end{remark}

\begin{remark}
    In this work, we adopt a linear Growing-Window scheme defined by $W_t=\lfloor\gamma t\rfloor$, where $\gamma\in (0,1]$ is a fixed constant. In particular, setting $\gamma=1$ recovers the traditional Monte Carlo (MC) method as a special case.
\end{remark}

Having defined the windowed counts and per-pair estimates, we now summarize each step $h$ with a statistically most informative representative pair drawn from the same window:
\begin{equation*}
    (s_h^{\star},a_h^{\star})= \underset{(s',a') \in \mathcal{S} \times \mathcal{A}}{\arg\max} \; N_{h}^{(W_t),t}(s',a').
\end{equation*}
Using the representative rather than pooling all pairs avoids diluting the estimate with many low-frequency pairs, reduces computation and storage, and, within a stationary environment, retains asymptotic correctness as the window grows. We further consider two `extreme' estimation strategies, namely pessimistic and optimistic versions for any $(s,a)$,
\begin{equation*}
\label{opt-b}
\overline{\alpha}_{i,h}^t(s,a):=\hat{\alpha}_{i,h}^{t-1}(s,a) + \zeta_h^{t-1}(s,a), \quad
\underline{\alpha}_{i,h}^t(s,a):=\hat{\alpha}_{i,h}^{t-1}(s,a) - \zeta_h^{t-1}(s,a),
\end{equation*}
and evaluate them at the representative pair,
\begin{equation}
\label{opt-a}
\overline{\alpha}_{i,h}^t:=\overline{\alpha}_{i,h}^t(s_h^{\star},a_h^{\star}), \quad
\underline{\alpha}_{i,h}^t:=\underline{\alpha}_{i,h}^t(s_h^{\star},a_h^{\star}).
\end{equation}

By construction, with probability $1-\delta$, any possible estimator $\hat{\alpha}_{i}$ satisfies inequality:  $\underline{\alpha}_{i}\leq \hat{\alpha}_{i} \leq \overline{\alpha}_{i}$. These optimistic and pessimistic variants are the upper and lower confidence bounds for the GW estimator while maintaining the minimal-data design.

\subsection{Feasibility and Strong Duality}\label{sec:Strong duality}

A policy $\pi$ is defined to be pessimistic if the expected cumulative constraint is no less than the pessimistic threshold $\overline{\alpha}$; i.e., $V_{g}^\pi\geq \overline{\alpha}$. Let $\mathcal{F}_{\textup{pes}}$ denote the feasibility set,
\begin{equation*}
    \mathcal{F}_{\textup{pes}}:=\{q\in\Delta{(\mathcal{M})}|\sum_{h,s,a}q_h^\pi(s,a;p)g_{i,h}(s,a)\geq \overline{\alpha}_i, \forall i\in [m]\}.
\end{equation*}
Similarly, we define a policy $\pi$ to be optimistic if the expected cumulative constraint is greater than or equal to the optimistic threshold $\underline{\alpha}$, i.e., $V_{g}^\pi\geq \underline{\alpha}$, and its feasibility set as below,
\begin{equation*}
    \mathcal{F}_{\textup{opt}}:=\{q\in\Delta{(\mathcal{M})}|\sum_{h,s,a}q_h^\pi(s,a;p)g_{i,h}(s,a)\geq \underline{\alpha}_i, \forall i\in [m]\}. 
\end{equation*}
\begin{remark}
From the definitions, pessimistic and optimistic policies are boundaries of all possible policies: $\mathcal{F}_{\textup{pes}}\subseteq \mathcal{F}_q\subseteq \mathcal{F}_\textup{opt}$ with probability $1-\delta$.
\end{remark}
Throughout, we make the following assumptions, which is standard in
the context of CMDPs \cite{altman1999constrained,efroni2020exploration,liu2021learning,liu2024learning}.
\begin{assumption}[Slater condition]\label{asp-1}
    There exists $\pi^0\in \Pi$ and $\xi \in \mathbb{R}_{>0}^m$ such that $V_{g_i}^{\pi^0} \ge \alpha_i +\xi_i$ for all $i\in[m]$. Set the Slater gap $\rho=(r^{\top}q^{\pi^{\star}}(p)-r^{\top}q^{\pi^0})/\min_{i\in[m]} \xi_i$.
\end{assumption}
\begin{wrapfigure}{r}{0.4\textwidth}
    \centering
    \includegraphics[width=0.37\textwidth]{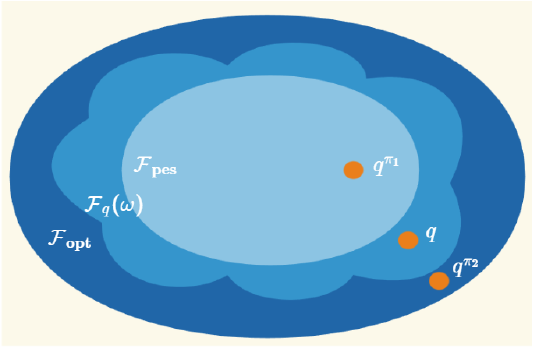}
    \caption{Feasible Sets and their Corresponding Feasible Solutions}
    \label{fig:feasible}
\end{wrapfigure}

Building on this assumption, we further require that the pessimistic tightening of the thresholds does not eliminate the existence of such a strictly feasible point. Similar `conservative tightening' assumptions have appeared in recent work \cite{bura2022dope}.

\begin{assumption}\label{asp-2}
    For $i\in[m]$, we have $\overline{\alpha}_i-\alpha_i<\xi_i$.
\end{assumption}

This condition ensures that the strictly feasible point from Assumption 1 remains feasible under the pessimistic setting. Their strong dualities are characterised by the following lemmas.

\begin{lemma}[Strong duality under pessimistic thresholds]
    Under Assumptions \ref{asp-1} and \ref{asp-2}, we have $\min_{\lambda}\max_{q \in \mathcal{F}_{\textup{pes}}} \mathcal{L}(q, \lambda) = \max_{q \in \mathcal{F}_{\textup{pes}}} \min_{\lambda}\mathcal{L}(q, \lambda).$
\end{lemma}
\begin{lemma}[Strong duality under optimistic thresholds]
    Under Assumption \ref{asp-1}, we obtain $\min_{\lambda}\max_{q \in \mathcal{F}_{\textup{opt}}} \mathcal{L}(q, \lambda) = \max_{q \in \mathcal{F}_{\textup{opt}}} \min_{\lambda}\mathcal{L}(q, \lambda).$  
\end{lemma}
Each lemma requires the existence of a policy that strictly satisfies the corresponding thresholds. We illustrate these feasible sets under pessimistic and optimistic thresholds and their corresponding feasible solutions in Figure \ref{fig:feasible}. If there exists a policy that strictly satisfies pessimistic thresholds, for all possible thresholds, the strong duality holds.

\subsection{A Primal-Dual Approach for CMDPs with Stochastic Thresholds}
\label{sec: SPOT}

In this section, we propose \textbf{SPOT} (\textit{{S}tochastic {P}essimistic-{O}ptimistic {T}hresholding}), a novel algorithm specifically tailored for CMDPs under stochastic thresholds. Unlike the closest algorithm OptPrimalDual-CMDP \cite{efroni2020exploration}, which assumes a fixed and known threshold, SPOT tackles the more challenging setting in which the threshold is unknown, and potentially changing over time.

\begin{algorithm}[htp]
\caption{\textbf{SPOT}}
\label{alg:spot}
\begin{algorithmic}[1]
\REQUIRE episodes $T$, $\bf\lambda_1 = 0$, stepsizes $\eta_t$ and $\eta_\lambda$, and initial policy $\pi_h^{1}(a|s)=1/A$ for all $s$, $a$, $h$.
\FOR{$t = 1, \dots, T$} 
    \STATE Update $\overline{r}^t$, $\overline{g}^t$, $\overline{p}^t$ via Equation \eqref{opt-p} and compute $\alpha_{\diamond}^t$ via Equation \eqref{opt-a} depending on the setting
    \STATE Truncated policy evaluation (Algorithm \ref{alg:truncated_policy}) for $\overline{r}^t$ and $\overline{g}^t$:
    \STATE \qquad$\hat{Q}_h^{t},\hat{V}_{\overline{g}}^{\pi_t}:= \text{TPE} (\overline{r}^t, \overline{g}^{t}, \overline{p}^{t}, \pi^t,\lambda_t)$
    \STATE Update primal variables for all $h$, $s$, $a\in [H] \times \mathcal{S} \times \mathcal{A}$:
        \STATE \qquad$\pi_h^{t+1}(a \mid s) = \frac{\pi_h^t(a \mid s) \exp(\eta_t \hat{Q}_h^{t}(s,a))}{\sum_{a'} \pi_h^t(a' \mid s) \exp(\eta_t \hat{Q}_h^{t}(s,a'))}$

    \STATE Update dual variables:
    \STATE \qquad $\lambda_{t+1} = \Pi_{\mathcal{C}}\left(\lambda_t + \frac{1}{\eta_\lambda} (\alpha_{\diamond}^t-\hat{V}_{\overline{g}}^{\pi_t})\right)$

    \STATE Execute $\pi_t$ via Algorithm \ref{alg:interaction} and update counters and empirical model (i.e., $N^t, \hat{r}^t, \hat{g}^t, \hat{p}^t, \hat{\alpha}^t$) via Equation \eqref{empirical-p}
\ENDFOR
\end{algorithmic}
\end{algorithm}

\paragraph{SPOT algorithm} For a given dual variable, the algorithm first updates the primal policy and subsequently refines the dual value using the estimated value. In the first stage (Line 2), it updates the optimistic estimators $\overline{r}^t$ and $\overline{g}^t$. Then, it computes the estimator $\alpha_{\diamond}^t$ as prescribed in Equation \eqref{opt-a}. Importantly, depending on the problem context, estimator $\alpha_{\diamond}^t$ can be either pessimistic estimator $\overline{\alpha}^t$ or optimistic estimator $\underline{\alpha}^t$, thereby accommodating different degrees of uncertainty in the thresholds. Thus, at each episode $t$, it yields a new CMDP $\mathcal{M}_t := (\mathcal{S}, \mathcal{A}, \overline{r}^t, \overline{g}^t, \alpha_{\diamond}^t, \overline{p}^{t})$, which the algorithm then optimises via 
$$\max_{q} \;\min_{\lambda \in \mathbb{R}^m_{\ge 0}} \mathcal{L}'(q, \lambda):=\overline{r}^{\top}q+\lambda^{\top}(\overline{G}^{\top}q-\alpha_{\diamond}).$$
In the policy evaluation stage (Line 3-4), a truncated policy evaluation approach (Algorithm \ref{alg:truncated_policy}) computes $Q$-values for both reward and constraints under policy $\pi_t$. Then, SPOT applies a Mirror Ascent (MA) rule to update the policy based on the computed Q-values. In the later stage of updating dual variables (Line 8), it employs a projected subgradient method:
\begin{equation}\label{eq:projection}
    \lambda_{t+1} = \Pi_{\mathcal{C}}\left(\lambda_t - \frac{1}{\eta_\lambda} (\overline{G}_{t-1}^{\top} q^{\pi_t} - \alpha_{\diamond}^t)\right),
\end{equation}
where $\overline{G}_{t-1}^{\top} q^{\pi_t}$ is equal to $\hat{V}_{\overline{g}}^{\pi_t}$ and the projection set is defined as $\mathcal{C}:=\{\lambda\in \mathbb{R}^m:0\leq \lambda\leq \rho\bf1\}$ with $\rho$ being the Slater gap defined in Section \ref{sec:Strong duality} and the projection operator $\Pi_{\mathcal{C}}:\mathbb{R}^m\to \mathcal{C}$ is defined by $\Pi_{\mathcal{C}}(x) := \arg\min_{y \in \mathcal{C}} \|y - x\|_2$. The feasibility proof for Equation \eqref{eq:projection} is presented in Appendix \ref{app:feasibility}. Finally, the policy $\pi_t$ is executed to collect new data, which is then used to update the empirical models and threshold estimates for subsequent iterations (Line 9).

\paragraph{Key differences and their significance}
\begin{itemize}[leftmargin=12pt] \item \textbf{Stochastic and unknown thresholds vs.\ fixed thresholds.} SPOT can handle thresholds that are unknown and random, going beyond the scope of fixed-threshold methods, which fit in more realistic, critical scenarios. 
\item \textbf{Pessimistic or optimistic thresholds balancing exploration and safety.} 
SPOT allows using either pessimistic or optimistic thresholds, selected according to the application's risk tolerance: pessimistic thresholds ensure robust feasibility under worst-case assumptions, while optimistic thresholds enable more ambitious exploration when additional risk is acceptable. This twofold approach empowers SPOT to adapt to varying levels of constraint uncertainty.
\end{itemize}

\section{Theoretical Analysis}\label{sec:thms}
This section theoretically analyses SPOT from perspectives of reward regret and constraint violation. Intuitively, reward regret captures how much total reward an algorithm `loses' compared to the optimal policy in hindsight; constraint violation tracks how often (or how much) an algorithm violates constraints. Both metrics are essential for quantifying how effectively our algorithm maximises rewards while adhering to constraints.

\subsection{Reward Regret and Constraint Violation of a Pessimistic Policy}
When safety and strict adherence to constraints are paramount, employing \emph{pessimistic thresholds} helps mitigate worst-case risks. In Algorithm \ref{alg:spot}, this corresponds to selecting $\alpha_{\diamond}^t=\overline{\alpha}^t$ at each episode $t$. Under this choice, the reward regret and constraint violation satisfy the following bounds.


\begin{theorem}[Bounds for reward regret and constraint violation of pessimistic policy]\label{thm:pes}
    Let $S$ be the number of states, $A$ be the number of actions, $H$ be the horizon length, $m$ be the number of constraints, $\rho$ be the Slater gap under Assumptions \ref{asp-1} and \ref{asp-2}, $\gamma$ be the Growing-Window parameter, $T$ be the number of episodes and $\mathcal{N}$ denote the maximum number of non-zero
transition probabilities across all state-action pairs. Then with probability $1-\delta$, Algorithm \ref{alg:spot} under pessimistic thresholds obtains the following regret bounds:
    \begin{align*}
    \mathcal{R}_T(r) &\leq \tilde{\mathcal{O}} \left( \sqrt{S \mathcal{N} H^4 T} + \sqrt{H^4(1 + m\rho)^2 T}+\sqrt{m^2\rho^2SAH^2T/\gamma} + (\sqrt{\mathcal{N}} + H) H^2 S A \right),\\
    \mathcal{R}_T(g) &\leq \tilde{\mathcal{O}} \left( \frac{1}{\rho} \left( \sqrt{ S \mathcal{N} H^4 T} + (\sqrt{\mathcal{N}} + H)  H^2 S A \right) + m \sqrt{SAH^4 T/\gamma}\right),
\end{align*}
where $\tilde{\mathcal{O}}$ hides polylogarithmic factors in $(S,A,H,T,1/\delta)$.
\end{theorem}

\begin{remark}
    The closest works would be Efroni et al. \cite{efroni2020exploration}, who first introduce online methods that achieve $\tilde{\mathcal{O}}(\sqrt{T})$ regret and constraint violation in CMDPs; later, Ghosh et al. \cite{ghosh2024towards} prove that $\tilde{\mathcal{O}}(\sqrt{T})$ regret and hard violation can be achieved by a primal-dual approach in tabular CMDPs. 
This paper removes the fixed-threshold assumption, and achieves comparable guarantees. 
\end{remark}
\begin{remark}
    Existing studies have also employed pessimistic thresholds, but the thresholds are predetermined 
\cite{bura2022dope}. In contrast, our approach adaptively updates the thresholds during learning.
\end{remark}

\paragraph{Proof sketch} We present a proof sketch here, with a detailed proof deferred to Appendix \ref{app:main theorem}.
%

\paragraph{Step 1} Decompose the sum of regret and violation scaled by dual variable, applying Lemmas \ref{lem:dual optimism} and \ref{lem:primal-dual update}, as follows, 
\begin{equation}\label{eq:decomposition}
    \mathcal{L}(\lambda)\leq  \sum_{t=1}^{T'} \left(f_t - \tilde{f}_t\right)+ \sum_{t=1}^{T'}\left(\tilde{f}_t + \lambda_t^{\top} \tilde{y}_t - f_{\pi^{\star}} - \lambda_t^{\top} y_{\pi^{\star}}\right)+\frac{1}{2\eta_\lambda} \sum_{t=1}^{T'} \|\tilde{y}_t\|^2 +\frac{\eta_\lambda}{2} \|\lambda\|_2^2,
\end{equation}
where $\mathcal{L}(\lambda)$ denotes the Lagrangian for any $0\leq \lambda \leq \rho \bf1$ and $\eta_\lambda$ denotes the learning rate for dual update. The first term on the right side measures the gap between actual reward and the optimistic estimates; the second term quantifies the cumulative error between the estimated policy and the optimal policy with $\lambda_t$; the last two terms capture the error from updating the dual variables. 

\paragraph{Step 2} Upper bound each term in the decomposition. For the first term and third term, we have
$\frac{1}{2\eta_\lambda} \sum_{t=1}^{T'} \|\tilde{y}_t\|^2 \leq \frac{m H^2 T}{2\eta_\lambda}$. By applying Lemma \ref{lem:policy errors-truncated}, we have $\left| \sum_{t=1}^{T'} \left(f_t - \tilde{f}_t\right) \right| \leq \tilde{\mathcal{O}} \left( \sqrt{S \mathcal{N} H^4 T}\right)$. 
Bounding the second term is one of major technical contributions in this paper. We obtain the following lemma which is proved in Appendix \ref{app:optimality}.  
\begin{lemma}[Policy optimality of pessimistic policy]\label{lem:pes-opt}
    Conditioning on the success event (formally defined in Appendix \ref{app:success event}), for any episode $t \in [T']$, parameter $\gamma\in(0,1]$. Under pessimistic thresholds, we have
\[\sum_{t=1}^{T'} \tilde{f}_t + \lambda_t^{\top} \tilde{y}_t - f_{\pi^{\star}} - \lambda_t^{\top} y_{\pi^{\star}} 
\leq \tilde{\mathcal{O}}\left(\sqrt{H^4(1 + m\rho)^2 T}+\sqrt{m^2\rho^2SAH^2T/\gamma}\right).\]
\end{lemma}
To obtain this, we need the following lemma.
\begin{lemma}[Martingale difference term bound]\label{lem:4}
    Let $\{\mathcal{F}_{h}^t\}_{(t,h)\in[T]\times[H]}$ be the filtration generated by the state-action sequences and threshold values observed up to time step $h$ at episode $t$. Suppose that for every $(t,h)$, $\zeta_h^t(s_h,a_h)$ satisfies
$\left|\zeta_h^t(s_h,a_h)-\mathbb{E}[\zeta_h^t(s_h,a_h)]|\mathcal{F}_{h}^{t-1}\right|\leq 2.$
Then, with probability $1-\delta$, 
\[
\left|\sum_{t=1}^T\sum_{h=1}^H\zeta_h^t(s_h,a_h)-\mathbb{E}\left[\zeta_h^t(s_h,a_h)|\mathcal{F}_{h}^{t-1}\right]\right| \leq \tilde{\mathcal{O}}\left(\sqrt{TH\text{ln}\frac{1}{\delta}}\right).
\]
\end{lemma}

\paragraph{Step 3} Combining Step 2 and Step 3, and after rearrangement, 
we have 
\begin{align*}
    \mathcal{L}(\lambda) \lesssim& \left(\rho + \frac{\|\lambda\|_2^2}{\rho}\right) \sqrt{mH^2T} + \left( \sqrt{H^4(1 + m\rho)^2 T}+\sqrt{m^2\rho^2SAH^2T/\gamma}\right)\\
    &+\sqrt{S \mathcal{N} H^4 T}+(\sqrt{\mathcal{N}}+H)H^2SA,
\end{align*}

Please note that this is an upper bound for the combination of reward and constraint violation. We can obtain the regret bound for reward or constraint violation through setting $\lambda=0$ or $\lambda_i=\rho e_i \mathbf{1}_{([\sum_{k=1}^{T'}g_{i,k}]_+\neq0)}$.


\subsection{Reward Regret and Constraint Violations of an Optimistic Policy}

When a certain degree of risk is acceptable in exchange for faster learning and higher reward potential (e.g., exploration‑heavy tasks or loosely constrained resource allocation), an optimistic threshold could be adopted in Algorithm \ref{alg:spot}, by setting $\alpha_\diamond^t=\underline{\alpha}^t$ at every episode $t$. Similarly, we prove reward regret and constraint violations of an optimistic policy as below.

\begin{theorem}[Bounds for reward regret and constraint violation with optimistic thresholds]\label{thm:opt}
    Under the same setting of Theorem \ref{thm:pes} and Assumption \ref{asp-1}, with probability $1-\delta$, Algorithm \ref{alg:spot} under optimistic thresholds achieves the following regret bounds:
\begin{align*}
\mathcal{R}_T(r) \leq& \tilde{\mathcal{O}} \left( \sqrt{S \mathcal{N} H^4 T} + \sqrt{H^4 (1 + m\rho')^2 T} + (\sqrt{\mathcal{N}} + H) H^2 S A \right),\\
\mathcal{R}_T(g) \leq& \tilde{\mathcal{O}} \left( \left(1 + \frac{1}{\rho'} \right) \left( \sqrt{m S \mathcal{N} H^4 T} + (\sqrt{\mathcal{N}} + H) \sqrt{m} H^2 S A\right) + m H^2\sqrt{T}\right)\\
&+\tilde{\mathcal{O}}\left(\sqrt{mSAH^2T/\gamma}+\sqrt{m}SAH\right).
\end{align*}
where $\rho'$ denotes the Slater gap under optimistic thresholds.
\end{theorem}

This theorem shows that, even under this aggressive stance, both reward regret and constraint violation remain sublinear over episodes $T$. 

\paragraph{Proof sketch} The proof idea is similar with the pessimistic counterpart, so we discuss the most significant differences below due to space limitation. A detailed proof is given in Appendix \ref{app:main theorem}.
\paragraph{Decomposition}
We can derive a decomposition similar to Equation \eqref{eq:decomposition}; however, unlike the pessimistic policy case, an additional term $\sqrt{\sum_{i=1}^{m} \left( \sum_{t=1}^{T'} (y_{t,i} - \tilde{y}_{t,i}) \right)^2} \|\lambda\|_2$ appears, which measures the total estimation error for all constraints. 
Applying Lemma \ref{lem:martingale}, we have that with probability $1-\delta$,
\begin{align*}
\sqrt{\sum_{i=1}^{m} \left( \sum_{t=1}^{T'} (y_{t,i} - \tilde{y}_{t,i}) \right)^2 }\leq \tilde{\mathcal{O}} \left( \sqrt{m S \mathcal{N} H^4 T} + (\sqrt{\mathcal{N}} + H+1) \sqrt{m}H^2SA+\sqrt{mSAH^2T/\gamma}\right).
\end{align*}

\paragraph{On-policy optimality} 
Optimistic policies also exhibit difference, as below.

\begin{lemma}[Policy optimality of optimistic policy]\label{lem:opt-opt}
    Conditioning on the success event, for any episode $t \in [T']$, if we adopt the optimistic thresholds, then
\[\sum_{t=1}^{T'} \tilde{f}_t + \lambda_t^{\top} \tilde{y}_t - f_{\pi^{\star}} - \lambda_t^{\top} y_{\pi^{\star}} \leq \tilde{\mathcal{O}}\left((1+m\rho')H^2\sqrt{T}\right).\]
\end{lemma}

\subsection{Safety-Performance Trade-Off}
\label{sec: comparisons}

Our theoretical analysis presents a {\it safety-performance trade-off}: policies aiming for strict adherence to constraints can incur larger reward regret, while policies that aggressively maximise rewards are more prone to constraint violations. Concretely, Theorems~\ref{thm:pes} and \ref{thm:opt} show that, compared to fixed-threshold baselines, pessimistic (strict) thresholds yield lower violations but higher reward regret, whereas optimistic thresholds exhibit the opposite behaviour.

\paragraph{Practical implications} 
One may wonder whether it is possible to tune a hyperparameter to leverage this trade-off without resorting to two extreme cases -- pessimistic or  optimistic. To this end, we may introduce a hyperparameter $\xi\in [0,1]$ to interpolate between these extremes. In particular, instead of setting $\alpha_{\diamond}^t=\overline{\alpha}^t$ or $\underline{\alpha}^t$, we  define a \emph{blended threshold}
$$\alpha_{\xi}^t=\xi\underline{\alpha}^t+(1-\xi)\overline{\alpha}^t$$
at episode $t$. Varying $\xi$ adjusts how `strict' or `relaxed' the thresholds are, thereby controlling the balance between safety and performance. While our theorems currently focus on the two boundary cases $\xi=0$ (pessimistic) and $\xi=1$ (optimistic), it is straightforward to extend to an analogous analysis for intermediate $\xi$ values. 

\section{Conclusion}

In this paper, we propose a novel primal-dual scheme for CMDPs with stochastic thresholds, which utilises samples from environmental interactions to estimate these thresholds. Building on this framework, we introduce SPOT, a primal-dual algorithm for episodic CMDPs with stochastic thresholds that is designed to balance regret minimisation with safety guarantees in an uncertain environment. Furthermore, we leverage either pessimistic or optimistic thresholds in SPOT, respectively. We theoretically prove that SPOT, under both threshold settings, achieves a regret bound of $\tilde{\mathcal{O}}(\sqrt{T})$ and a constraint violation of $\tilde{\mathcal{O}}(\sqrt{T})$ over $T$ episodes, which are in the same order as that of an approach with certain thresholds. Our findings reveal a subtle dance: pessimistic thresholds enforce stricter safety constraints that effectively reduce violations, albeit at the cost of increased reward regret; in the optimistic setting, a reverse situation is observed. This creates a `margin' between safety and performance without relying on prior threshold information. More fundamentally, our treatment of thresholds as stochastic variables opens up a new design space for developing more varied and sophisticated safety mechanisms. 



\bibliographystyle{plain} 
\bibliography{references} 

\clearpage


\appendix
\label{appendix}
\section{Summary of Notation}
In this section, we summarise the key notation used throughout the paper as follows.
\begin{table}[htp]\label{tab}
\centering
\caption{List of Notation}
\begin{tabular}{p{0.4\linewidth} p{0.5\linewidth}}
\toprule
\textbf{Description} & \textbf{Notation} \\
\midrule
State space & $\mathcal{S}$, with cardinality $S$ \\ 
Action space & $\mathcal{A}$, with cardinality $A$ \\ 
Number of constraints & $m$ \\
Time horizon & $H$ \\
Transition probability & $p_h(s' \mid s, a)$ \\
Initial state & $s_1 \in \mathcal{S}$ \\
Slater gap of $\pi^0$ & $\rho \;=\; (r^{\top}q^{\pi^{\star}}(p)-r^{\top}q^{\pi^0})/\min_{i \in [m]} \xi_i$ \\
Number of episodes & $T$ \\
Reward &  $\tilde{r}_h(s,a)$ with $\mathbb{E}[\tilde{r}_h(s,a)] = r_h(s,a)$ \\
Constraint & $\tilde{g}_{i,h}(s,a)$ with $\mathbb{E}[\tilde{g}_{i,h}(s,a)] = g_{i,h}(s,a)$ \\
Threshold & $\tilde{\alpha}_{i,h}(s,a)$ with $\mathbb{E}[\tilde{\alpha}_{i,h}(s,a)] = \alpha_{i,h}$ \\
Constraint functions & $y_t(s,a) = \alpha^t- Gq^{\pi_t}(p)$ \\
Policy & $\pi \in \Pi$\\
Value function & 
$V^\pi_{v,h}(s) = \mathbb{E}\Bigl[\sum_{h'=h}^{H} v_{h'}(s_{h'}, a_{h'}) \mid s_h = s\Bigr]$ \\
Value function (vector-valued) & 
$V^\pi_{v} = \bigl(V^\pi_{v,1}(s_1), \dots, V^\pi_{v,H}(s_1)\bigr)^\top \in \mathbb{R}^m$ \\
$Q$-value function & 
$Q^\pi_{v,h}(s,a) = \mathbb{E}\Bigl[\sum_{h'=h}^{H} v_{h'}(s_{h'}, a_{h'}) \mid s_h=s,\, a_h=a\Bigr]$ \\
Occupancy measure & $q_h^\pi(s,a;p) = \mathbb{P}[s_h = s,\, a_h = a\mid p,\pi]$ \\
Occupancy measure (matrix form) & $q=[q_h^\pi(s,a;p)]_{H\times S\times A}$ \\
Lagrangian under stochastic thresholds & $\mathcal{L}(q, \lambda):=r^{\top}q+\lambda^{\top}(G^{\top}q-\alpha)$ \\
Lagrangian of Algorithm \ref{alg:spot} & $\mathcal{L'}(q, \lambda):=\overline{r}^{\top}q+\lambda^{\top}(\overline{G}^{\top}q-\alpha_{\diamond})$\\
Optimal policy & $q^{\pi^\star} \in \arg\max_{q\in \mathcal{F}_q} \min_{\lambda \in \mathbb{R}_{\ge 0}^m} \mathcal{L}(q,\lambda)$ \\
Confidence level & $1 - \delta$ \\
Reward regret & $\mathcal{R}_T(r) = \sum_{t \in [T]} \left( V_r^\star - V_r^{\pi_t} \right)$ \\
Constraint violation & $\mathcal{R}_T(g) = \max_{i \in [m]} \sum_{t \in [T]} \left( \alpha_i - V_{g_i}^{\pi_t} \right)$ \\
Step size & $\eta_{\lambda} > 0$ (hyperparameter) \\
Dual domain & $\mathcal{C} = [0, \rho\textbf{1}]^m$ \\
Visitation counter & $N_{t,h}(s,a) = \sum_{l=1}^{t} \mathbf{1}\{s_h^l = s, a_h^l = a\}$ \\
Averages & 
$\hat{r}_{t,h}(s,a),\; \hat{g}_{t,h}(s,a),\; \hat{\alpha}_{t,h}(s,a)$\\
Exploration bonuses & 
$\phi_{h}^{r,t},\; \phi_{h}^{g,t},\; \zeta_{h}^{t}$ \\
Optimistic estimates & 
$\overline{r}_t,\; \overline{g}_t,\; \overline{\alpha}_t$ \\
Pessimistic estimates & 
$\underline{\alpha}^t$ \\
Scenario-dependent threshold & $\alpha_{\diamond}$\\
Success event & $\mathcal{E}$ \\
Truncated $Q$-function & 
$\hat{Q}^{t}_{h}(s,a) = \hat{Q}^{t}_{h}(\overline{r}_t,s,a) + \sum_i \lambda_{t,i} \hat{Q}^{t}_{h}(\overline{g}_{t,i},s,a)$ \\
Truncated value function & 
$\hat{V}^{t}_{v_t,h}(s) = \langle \pi_{t,h}(\cdot \mid s), \hat{Q}^{t}_{v_t,h}(s, \cdot) \rangle$ \\
\bottomrule
\end{tabular}
\end{table}

\section{Lagrangian Formulation of CMDPs}

This section discusses the Lagrangian formulation of CMDPs with and without stochastic thresholds.

\subsection{Lagrangian Formulation of CMDPs without Stochastic Thresholds}

Under the bandit feedback paradigm, traditional CMDPs usually assume that rewards and constraints are stochastic or adversarial. Here, we take the stochastic case as an example, that is, rewards and constraints are randomly drawn from some probability distributions $\mathcal{R}$ and $\mathcal{G}_i$ for all  $i\in [m]$. 
The immediate reward after taking an action $a$ at state $s$ is a random variable $\tilde{r}_h(s,a)\in [0,1]$ with expectation $\mathbb{E}_\mathcal{R}[\tilde{r}_h(s,a)]=r_h(s,a)$. Similarly, for every constraint $i\in [m]$, the immediate constraint $\tilde{g}_{i,h}(s,a)\in [0,1]$ for every state-action pair $(s,a)$ is a random variable with expectation $\mathbb{E}_{\mathcal{G}_i}[\tilde{g}_{i,h}(s,a)]=g_{i,h}(s,a)$.

Now, we explore strong duality in this setting. Occupancy measure is defined as in Section \ref{sec: New Setting}. We can reformulate the saddle-point problem into the following optimisation problem:
\begin{equation*}
\label{objective2}
\max_{q \in \Delta(\mathcal{M})} \;\min_{\lambda \in \mathbb{R}^m_{\ge 0}} \mathcal{L}(q, \lambda):=r^{\top}q+\lambda^{\top}(G^{\top}q-\alpha),
\end{equation*}
where $\Delta(\mathcal{M})$ is the set of valid occupancy measures, $r$ is the
reward vector, and $G$ is the constraint matrix. Without the occupancy measure, the feasible set can be characterised as:
$\mathcal{F}_\pi:=\left\{\pi \in \Pi| V_{g_i}^\pi\geq \alpha_i, \forall i \in [m]\right\}$. Combining the definition of occupancy measure, the feasible set is described by:
$$\mathcal{F}_q:=\{q\in\Delta(\mathcal{M})|\sum_{h,s,a}q_h^\pi(s,a;p)g_{i,h}(s,a)\geq \alpha_i, \forall i\in [m]\}.$$
In order to further discuss its strong duality, we make the following assumption, which is commonly used in the study of CMDPs \cite{altman1999constrained, efroni2020exploration,li2024faster,ying2022dual,ding2022convergence,ding2023provably,muller2024truly,paternain2022safe}.

\begin{customasp}{1}[Slater condition]
\label{asp-1-1}
There exists $\pi^0\in \Pi$ and $\xi \in \mathbb{R}_{>0}^m$ such that $V_{g_i}^{\pi^0} \ge \alpha_i +\xi_i$ for all $i\in[m]$. Set the Slater gap $\rho=(r^{\top}q^{\pi^{\star}}(p)-r^{\top}q^{\pi^0})/\min_{i\in[m]} \xi_i$. 
\end{customasp}

Then, one can prove that CMDPs satisfy strong duality as follows.

\begin{lemma}[Strong duality under fixed thresholds \cite{paternain2019constrained}]
    We have
\begin{equation*}
\min_{\lambda \in \mathbb{R}_{\geq 0}^m} 
\;\max_{\pi \in \Pi}
~\mathcal{L}(\pi, \lambda)
~
=~
\max_{\pi \in \Pi} 
\;\min_{\lambda \in \mathbb{R}_{\geq 0}^m}
~\mathcal{L}(\pi, \lambda).
\end{equation*}
and both optima are attained.
\end{lemma}

\subsection{Lagrangian Formulation of CMDPs with Stochastic Thresholds}
To rigorously study the implications of stochastic thresholds, we explore two distinct boundary scenarios: \emph{pessimistic} and \emph{optimistic}. Under the pessimistic scenario, a policy $\pi$ is defined to be feasible if it satisfies constraints under a conservative threshold $\overline{\alpha}$: $V_{g}^{\pi} \geq \overline{\alpha}$.
We denote the feasibility set for pessimistic policies as:
\begin{equation*}
    \mathcal{F}_{\textup{pes}} := \{ q \in \Delta(\mathcal{M}) \mid \sum_{h,s,a} q_h^\pi(s,a;p)g_{i,h}(s,a) \geq \overline{\alpha}_i, \forall i \in [m] \}.
\end{equation*}
In contrast, under the optimistic scenario, feasibility is evaluated against a more lenient threshold $\underline{\alpha}$: $V_{g}^{\pi} \geq \underline{\alpha}$.
The corresponding feasibility set for optimistic policies is defined as:
\begin{equation*}
    \mathcal{F}_{\textup{opt}} := \{ q \in \Delta(\mathcal{M}) \mid \sum_{h,s,a} q_h^\pi(s,a;p)g_{i,h}(s,a) \geq \underline{\alpha}_i, \forall i \in [m] \}.
\end{equation*}
From these definitions, it follows that the pessimistic and optimistic scenarios define the boundaries for all feasible policies, yielding:
$\mathcal{F}_{\textup{pes}} \subseteq \mathcal{F}_q \subseteq \mathcal{F}_{\textup{opt}}$ with probability $1 - \delta$ due to Theorem \ref{thm:estimator2}.

Our theoretical analysis follows Assumption \ref{asp-1-1}. Building on this standard assumption, we further require that the pessimistic tightening of the thresholds does not eliminate the existence of such a strictly feasible point. Similar “conservative tightening” assumptions have appeared in recent studies \cite{bura2022dope}.
\begin{customasp}{2}[Slater condition]
\label{asp-2-1}
For $i\in[m]$, we have $\overline{\alpha}_i-\alpha_i<\xi_i$. 
\end{customasp}
Then we have the following lemmas:
\begin{customlem}{1}[Strong duality under pessimistic thresholds]
    Under Assumptions \ref{asp-1-1} and \ref{asp-2-1}, we have $\min_{\lambda}\max_{q \in \mathcal{F}_{\textup{pes}}} \mathcal{L}(q, \lambda) = \max_{q \in \mathcal{F}_{\textup{pes}}} \min_{\lambda}\mathcal{L}(q, \lambda).$
\end{customlem}
\begin{proof}
Under Assumption Assumptions \ref{asp-1-1} and \ref{asp-2-1}, we solve the following convex optimisation problem:
\begin{equation}\label{eq:convex}
    q^{\pi^{\star}}\in \arg\max_{q \in \mathcal{F}_\text{pes}} r^{\top} q \quad \text{s.t.} \quad G^\top q \geq \overline{\alpha},
\end{equation}
which satisfies all conditions in Assumption \ref{assumption:convex}. Specifically:
\begin{enumerate}
    \item $\mathcal{F}_\text{pes}$ is a polytope and hence convex.
    \item The objective function \(f(\cdot) = -r^{\top}(\cdot)\) is affine and therefore convex.
    \item Each constraint \(y_i(\cdot) := \overline{\alpha}_i - g_i^{\top}(\cdot)\) is affine and therefore convex.
    \item Equation \eqref{eq:convex} admits at least one feasible solution, and since $\mathcal{F}_\text{pes}$ is compact while $f(\cdot)$ is continuous, an optimal solution exists.
    \item A Slater point exists.
    \item All dual problems admit an optimal solution because the domain is compact and $f(\cdot)+\lambda^{\top}y(\cdot)$ is continuous.
\end{enumerate}
Thus, the result immediately follows from Lemma \ref{lem:convex}.
\end{proof}

\begin{customlem}{2}[Strong duality under optimistic thresholds]
    Under Assumption \ref{asp-1-1}, we have $\min_{\lambda}\max_{q \in \mathcal{F}_{\textup{pes}}} \mathcal{L}(q, \lambda) = \max_{q \in \mathcal{F}_{\textup{pes}}} \min_{\lambda}\mathcal{L}(q, \lambda).$ 
\end{customlem}
The proof is similar to that under the pessimistic thresholds.



\section{Estimators of Thresholds}\label{app:estimation}

We propose two types of threshold estimators, a Monte Carlo estimator and a Growing-Window (GW) estimator as discussed in the main text. 

\paragraph{Monte Carlo estimator}
For every constraint $i \in [m]$, state $s$, action $a$, step $h\in[H]$ and episode $t\in[T]$, let $(s_h^l, a_h^l)$ denotes the state-action pair visited in episode $l$ at step $h$; $(s_h^l, a_h^l, s_{h+1}^l)$ represents the state-action pair $(s_h^l, a_h^l)$ is visited and the environment evolves to next state $s_{h+1}^l$ at step $h$ in episode $l$; $\mathbf{1}_{X}$ is the indicator function of $X$; and $N_h^t(s,a)=\sum_{l=1}^{t-1}\mathbf{1}_{\{s_h^l=s,a_h^l=a\}}$ is the total number of visits to the pair $(s,a) \in \mathcal{S} \times \mathcal{A}$ at step $h$ before episode $t\in [T]$. First, we give the empirical averages of the thresholds as follows:
\begin{equation*}
\hat{\alpha}_{i,h}^t(s,a) := \frac{\sum_{l=1}^t \tilde{\alpha}_{i,h}^l(s,a)\,\mathbf{1}_{\{s_h^l=s,a_h^l=a\}}}{\max\{1, N_h^{t}(s,a)\}}.
\end{equation*}

Then, we have the following lemmas. 

\begin{lemma}[Confidence interval of MC thresholds]\label{lem:confidence interval} 
    Given any confidence parameter $\delta \in (0, 1)$, constraint $i \in [m]$, step $h \in [H]$, episode $t \in [T]$ and state-action pair $(s,a) \in \mathcal{S} \times \mathcal{A}$, with probability at least $1 - \delta$, it holds
    \[
\left| \hat{\alpha}_{i,h}^t(s,a) - \alpha_{i,h} \right| \leq \iota_h^t(s,a),
\]
\noindent
where $\iota_h^t(s,a) := \sqrt{\frac{\ln\left(\frac{2}{\delta}\right)}{2N_t(s,a)}}$.
\end{lemma}

\begin{proof} 
Focus on specifics $i \in [m]$, $t \in [T]$ and $(s,a) \in \mathcal{S} \times \mathcal{A}$. By Hoeffding's inequality and noticing that threshold values are bounded in $[0, H]$, it holds that:
\[
\mathbb{P}\left[\left|\hat{\alpha}_{i,h}^t(s,a) - \alpha_{i,h}  \right| \geq \frac{c}{N_t(s,a)}\right] \leq 2\exp\left(-\frac{2c^2}{N_t(s,a)}\right).
\]
Setting $\delta = 2\exp\left(-\frac{2c^2}{N_t(s,a)}\right)$ and solving to find a proper value of $c$ gives the result for the threshold function.
\end{proof}

\begin{theorem}[Asymptotically consistency of MC threshold empirical estimator]\label{lem:estimation error-thre}
    Given a confidence parameter $\delta \in (0,1)$, with probability at least $1 - \delta$, the following holds for every constraint $i \in [m]$, step $h \in [H]$, episode $t \in [T]$, and state-action pair $(s,a) \in \mathcal{S} \times \mathcal{A}$:
\[
\left| \hat{\alpha}_{i,h}^t(s,a) - \alpha_{i,h} \right| \leq \zeta_h^t(s,a),
\]
where $\zeta_h^t(s,a) := \min \left\{ 1, \sqrt{\frac{4 \ln (mSAHT/ \delta)}{\max \{1, N_t(s,a) \}}} \right\}$.
\end{theorem}

\begin{proof}
    From Lemma \ref{lem:confidence interval}, given $\delta' \in (0, 1)$, we have for any constraint $i\in [m]$, step $h \in [H]$, episode $t \in [T]$ and state-action pair $(s,a) \in \mathcal{S} \times \mathcal{A}$:
\[
\mathbb{P}\left[ \left| \hat{\alpha}_{i,h}^t(s,a) - \alpha_{i,h} \right| \leq \iota_h^t(s,a) \right] \geq 1 - \delta'.
\]
Now, we are interested in the intersection of the aforementioned events:
\[
\mathbb{P}\left[ \bigcap_{s,a,m,h,t} \left\{ \left| \hat{\alpha}_{i,h}^t(s,a) - \alpha_{i,h} \right| \leq \iota_h^t(s,a)\right\} \right].
\]
Thus, we have:
\begin{align*}
   &\mathbb{P}\left[\bigcap_{s,a,m,h,t} \left\{ \left| \hat{\alpha}_{i,h}^t(s,a) - \alpha_{i,h} \right| \leq \iota_h^t(s,a) \right\} \right]\\
   =& 1 - \mathbb{P}\left[ \bigcup_{s,a,m,h,t}  \left\{ \left| \hat{\alpha}_{i,h}^t(s,a) - \alpha_{i,h} \right| \leq \iota_h^t(s,a) \right\}^c \right]\\
   =& 1 - \sum_{s,a,m,h,t}\mathbb{P}\left[ \left\{ \left| \hat{\alpha}_{i,h}^t(s,a) - \alpha_{i,h} \right| \leq \iota_h^t(s,a) \right\}^c \right]\\
   \geq& 1 - mSAHT \delta',
\end{align*}
where the inequality holds by union bound. Noticing that $\alpha_{i,h}^t(s,a) \leq H$, we substitute $\delta'$ with $\delta := \frac{\delta'}{mSAHT}$ in $\iota_t(s,a)$ with an additional union bound over the possible values of $N_t(s,a)$. This yields $\zeta_t(s,a)$, and concludes the proof.
\end{proof}

\paragraph{Growing-Window estimator}

For each constraint $i \in [m]$, state-action pair $(s,a)\in \mathcal{S}\times\mathcal{A}$, step $h\in[H]$, and episode $t\in[T]$, let the sequence $\{W_t\}_{t=1}^T$ denote a strictly increasing window-size parameter (e.g., $W_t=\lfloor \gamma t\rfloor$ for some constant $\gamma \in (0,1]$). Define
\[
N_h^{(W_t),t-1}(s,a)=\sum_{l=\max\{1,t-W_t+1\}}^{t-1}\mathbf{1}_{\{s_h^l=s,a_h^l=a\}}
\]
as the number of visits to $(s,a)$ at step $h$ during the most recent $W_t$ episodes before episode $t$.

Then, the Growing-Window empirical threshold estimator is:
\begin{equation*}
\hat{\alpha}_{i,h}^{(W_t),t}(s,a)=\frac{\sum_{l=\max\{1,t-W_t+1\}}^t\tilde{\alpha}_{i,h}^l(s,a)\mathbf{1}_{\{s_h^l=s,a_h^l=a\}}}{\max\{1,N_h^{(W_t),t}(s,a)\}}.
\end{equation*}

We then have the following theory.

\begin{lemma}[Confidence interval of GW thresholds]\label{lem:gw_confidence}
Given any $\delta\in(0,1)$, constraint $i\in[m]$, step $h\in[H]$, episode $t\in[T]$, and state-action pair $(s,a)\in \mathcal{S}\times \mathcal{A}$, with probability at least $1-\delta$, we have:
\[
\left|\hat{\alpha}_{i,h}^{(W_t),t}(s,a)-\alpha_{i,h}\right|\leq\iota_h^{(W_t),t}(s,a),
\]
where $\iota_h^{(W_t),t}(s,a)=\sqrt{\frac{\ln\left(\frac{2}{\delta}\right)}{2\max\{1,N_h^{(W_t),t-1}(s,a)\}}}$.
\end{lemma}

\begin{proof}
For fixed indices $i,h,t,(s,a)$, applying Hoeffding's inequality and noting threshold values are bounded within $[0,H]$, we obtain:
\[
\mathbb{P}\left[\left|\hat{\alpha}_{i,h}^{(W_t),t}(s,a)-\alpha_{i,h}\right|\geq\epsilon\right]\leq2\exp\left(-2\epsilon^2N_h^{(W_t),t-1}(s,a)\right).
\]
Setting $\delta=2\exp\left(-2\epsilon^2N_h^{(W_t),t-1}(s,a)\right)$ and solving for $\epsilon$ yields the stated confidence bound.
\end{proof}

\begin{customthm}{1}[Asymptotic consistency of GW threshold estimator]\label{lem:gw_consistency}
Given $\delta\in(0,1)$, with probability at least $1-\delta$, the following holds uniformly for each constraint $i\in[m]$, step $h\in[H]$, episode $t\in[T]$, and state-action pair $(s,a)\in\mathcal{S}\times\mathcal{A}$:
$$
\left|\hat{\alpha}_{i,h}^{(W_t),t}(s,a)-\alpha_{i,h}\right|\leq\zeta_h^{(W_t),t}(s,a),
$$
where $\zeta_h^{(W_t),t}(s,a)=\min\left\{1,\sqrt{\frac{4\ln(mSAHT/\delta)}{\max\{1,N_h^{(W_t),t-1}(s,a)\}}}\right\}$.
\end{customthm}

\begin{proof}
By Lemma \ref{lem:gw_confidence}, for any given confidence level $\delta'$, we have:
\[
\mathbb{P}\left[\left|\hat{\alpha}_{i,h}^{(W_t),t}(s,a)-\alpha_{i,h}\right|\leq\iota_h^{(W_t),t}(s,a)\right]\geq1-\delta'.
\]
Taking a union bound over all possible choices of $i\in[m]$, $h\in[H]$, $t\in[T]$, and $(s,a)\in\mathcal{S}\times\mathcal{A}$, we have:
\begin{align*}
&\mathbb{P}\left[\bigcap_{i,h,t,s,a}\left\{\left|\hat{\alpha}_{i,h}^{(W_t),t}(s,a)-\alpha_{i,h}\right|\leq\iota_h^{(W_t),t}(s,a)\right\}\right]\geq1-mSAHT\delta'.
\end{align*}
Letting $\delta'=\frac{\delta}{mSAHT}$ and substituting into $\iota_h^{(W_t),t}(s,a)$, we derive the stated uniform bound with probability at least $1-\delta$. This completes the proof.
\end{proof}
\begin{remark}
When there is no ambiguity, we simplify the notation $\hat{\alpha}_{i,h}^{(W_t),t}(s,a)$ to $\hat{\alpha}_{i,h}^{t}(s,a)$ and $\zeta_h^{(W_t),t}(s,a)$ to $\zeta_h^{t}(s,a)$. In this paper, we employ a linear Growing-Window scheme specified by $W_t=\lfloor\gamma t\rfloor$, for a fixed constant $\gamma\in (0,1]$. In particular, when $\gamma=1$, we recover the classic MC method as a special case.
\end{remark}

Furthermore, we consider both pessimistic and optimistic estimators for thresholds under conservative and aggressive explorations, respectively as below
\begin{align*}
\overline{\alpha}_{i,h}^t(s,a):=\hat{\alpha}_{i,h}^{t-1}(s,a) + \zeta_h^{t-1}(s,a),\tag{Pessimistic}\\
\underline{\alpha}_{i,h}^t(s,a):=\hat{\alpha}_{i,h}^{t-1}(s,a) - \zeta_h^{t-1}(s,a).\tag{Optimistic}
\end{align*}

 We then give the empirical averages of the reward, constraints and transition probabilities as follows: 
 \begin{subequations}\label{empirical-p}
\begin{align}
    \hat{r}_h^t(s,a) &:= \frac{\sum_{l=1}^t \tilde{r}_{h}^l(s,a)\,\mathbf{1}_{\{s_h^l=s,a_h^l=a\}}}{\max\{1, N_h^{t}(s,a)\}}, \\
 \hat{g}_{i,h}^t(s,a) &:= \frac{\sum_{l=1}^t \tilde{g}_{i,h}^l(s,a)\,\mathbf{1}_{\{s_h^l=s,a_h^l=a\}}}{\max\{1, N_h^{t}(s,a)\}}, \\
 \hat{p}_h^t(s'\mid s,a) &:= \frac{\sum_{l=1}^t \mathbf{1}_{\{s_h^l=s,a_h^l=a,s_{h+1}^l=s'\}}}{\max\{1, N_h^{t}(s,a)\}}.
\end{align}
\end{subequations}
Next, we define optimistic estimators for the reward, constraints and transition probabilities as below,
\begin{subequations}\label{opt-p}
\begin{align}
\overline{r}_h^t(s,a)&:=\hat{r}_h^{t-1}(s,a)+\phi_h^{t-1}(s,a),\\
\overline{g}_{i,h}^t(s,a)&:=\hat{g}_{i,h}^{t-1}(s,a)+\phi_h^{t-1}(s,a), \\
\overline{p}_h^t(s'|s,a)&:=\hat{p}_h^{t-1}(s'|s,a).
\end{align}
\end{subequations}
The bonus term $\phi_h^t$ combines the uncertainties arising from both reward and transition estimations at step $h$ in episode $t$: $\phi_h^{t}(s,a) = \phi_{h}^{r,t}(s,a) + \phi_{h}^{p,t}(s,a)$, where the reward bonus $\phi_h^{r,t}(s,a)=\mathcal{O}\left(\sqrt{\frac{\text{ln}(SAHmT/\delta)}{\max\{1, N_h^t(s,a)\}}}\right)$ and the transition bonus $\phi_h^{p,t}(s,a)=\mathcal{O}\left(H\sqrt{\frac{S+\text{ln}(SAHT/\delta)}{\max\{1, N_h^t(s,a)\}}}\right)$ for any confidence parameter $\delta\in (0,1)$. 

\section{Success Event}
\label{app:success event}
Fixing a confidence parameter $\delta>0$ and defining $\delta':=\delta/4$, we first introduce the following \emph{failure events}:
\begin{align*}
   F_t^\alpha &:= \left\{ \exists s, a, h, i: \left| \hat{\alpha}_{i,h}^t(s,a) - \alpha_{i,h} \right| \geq \zeta_{h}^{t}(s,a) \right\},\\ 
   F_t^r &:= \left\{ \exists s, a, h : \left| \hat{r}_h^t(s,a) - r_h(s,a) \right| \geq \phi_{h}^{r,t}(s,a) \right\},\\
   F_t^g &:= \left\{ \exists s, a, h, i: \left| \hat{g}_{i,h}^t(s,a) - g_{i,h}(s,a) \right| \geq \phi_{i,h}^{g,t}(s,a) \right\},\\
   F_t^p &:= \left\{ \exists s, a, s', h : \left| p_h(s' \mid s,a) - \hat{p}_{h}^{t-1} (s' \mid s,a) \right| \geq \phi_{h}^{p,t}(s,a,s') \right\},\\
   F_t^N &:= \left\{ \exists s, a, h : N_{h}^{t-1}(s,a) \leq \frac{1}{2} \sum_{j<t} q_h^{\pi_j}(s,a ; p) - H \ln \left(\frac{SAH}{\delta'}\right) \right\}.
\end{align*}

Then, we define the union of these events over all episodes,
\begin{align*}
    F^{\alpha}&:=\bigcup_{t\in [T]}F_t^\alpha, \quad
    F^r:=\left(\bigcup_{t\in [T]}F_t^r\right)\bigcup\left(\bigcup_{t\in [T]}F_t^g\right),\\[1ex]
    F^{p}&:=\bigcup_{t\in [T]}F_t^p,\quad
    F^{N}:=\bigcup_{t\in [T]}F_t^N.
\end{align*}

Furthermore, the success event $\mathcal{E}$ is defined as the complement of those failure events:
\begin{equation*}
    \mathcal{E} = \overline{F^\alpha\cup F^r \cup F^p \cup F^N}.
\end{equation*}
We have the following lemma.

\begin{lemma}[Success event]
Setting $\delta' = \frac{\delta}{4}$, we have $\mathbb{P}[\mathcal{E}] \geq 1-\delta$.
\end{lemma}

\begin{proof}
We apply the union bound to each event separately.
\begin{itemize}[leftmargin=12pt]
    \item By Lemma \ref{lem:gw_consistency}, we have $\mathbb{P}[F^{\alpha}] \leq \delta'$.
    \item Using Hoeffding’s inequality and union bound arguments over all state-action-step combinations, similarly, we obtain $\mathbb{P}[F^r] \leq \delta'$.
    \item Using concentration inequalities for multinomial distributions \cite{maurer2009empirical} and the union bound, we derive $\mathbb{P}[F^p] \leq \delta'$.
    \item Employing similar techniques as in \cite{dann2017unifying}, by bounding occupancy measure deviations, we obtain $\mathbb{P}[F^N] \leq \delta'$.
\end{itemize}

Combining these results with the union bound, we have
$$\mathbb{P}[F^\alpha\cup F^r \cup F^p \cup F^N]\leq \mathbb{P}[F^\alpha]+\mathbb{P}[F^r]+\mathbb{P}[F^p]+\mathbb{P}[F^N]\leq 4\delta'=\delta.$$
Thus, $\mathbb{P}[\mathcal{E}]=1-\mathbb{P}[F^\alpha\cup F^r \cup F^p \cup F^N]\geq 1-\delta$. 

This completes the proof.
\end{proof}

\section{Policy Evaluation}\label{app:PE}
Truncated policy evaluation is essential in CMDPs 
under stochastic threshold settings. Given the presence of stochastic constraints and additional exploration bonuses, unbounded value estimates can lead to instability and hinder theoretical analysis. We employ truncation to maintain boundedness and numerical stability of value functions.

Formally, for given estimates of reward $\overline{r}_h(s,a)$, constraint functions $\overline{g}_{i,h}(s,a)$, transition probabilities $\overline{p}_h(\cdot\mid s,a)$, we iteratively compute truncated $Q$ and $V$ value estimates. The truncated Q-value update at each timestep $h$ is expressed as below,
\begin{equation*}
    \hat{Q}_h^{\pi}(s,a; \overline{l}, \overline{p}) = \min \left\{ \overline{l}_h(s,a) + \sum_{s'} \overline{p}_h(s'\mid s,a) \hat{V}_{\overline{l},h+1}^{\pi}(s'),\;H-h+1 \right\},
\end{equation*}
where $\overline{l}_h(s,a)$ denotes the generalized immediate payoff (reward or cost with bonus), and $\hat{V}_h^{\pi}(s; \overline{l}, \overline{p})$ denotes the truncated value function,
$$\hat{V}_{\overline{l},h}^{\pi}(s) = \Big\langle \hat{Q}_h^{\pi}(s,a; \overline{l}, \overline{p}), \pi_h(a \mid s) \Big\rangle.$$
The detailed truncated policy evaluation algorithm is shown in Algorithm \ref{alg:truncated_policy}.
\begin{algorithm}[H]
\caption{TPE (Truncated Policy Evaluation \cite{efroni2020exploration})}
\label{alg:truncated_policy}
\begin{algorithmic}[1]
\REQUIRE  estimate $\overline{r}_h^t(s,a)$, $\overline{g}_h^t(s,a)$, transition probability $\overline{p}_h^t(s' \mid s,a)$, policy $\pi_h(a \mid s)$, and initial value function $\hat{V}_{\overline{r},H+1}^{\pi_t}(s) =\hat{V}_{\overline{g},H+1}^{\pi_t}(s)=0$ for all $s$.
\FOR{$h = H, H-1,\dots, 1$}
    \FOR{$ (s, a) \in \mathcal{S} \times \mathcal{A}$}
    \STATE Compute truncated Q-function:
        \STATE \quad $\hat{Q}_h^{\pi_t}(s,a; \overline{r}, \overline{p}) = \min \Big\{ \overline{r}_h(s,a) + \overline{p}_h(\cdot \mid s,a) \hat{V}_{\overline{r},h+1}^{\pi_t}(\cdot), H-h+1 \Big\}$
        \STATE \quad $\hat{Q}_h^{\pi_t}(s,a; \overline{g}_{i}, \overline{p}) = \min \Big\{ \overline{g}_{i,h}(s,a) + \overline{p}_h(\cdot \mid s,a) \hat{V}_{\overline{g},h+1}^{\pi}(\cdot), H-h+1 \Big\}$ \quad ($\forall i \in [m]$)
        \STATE \quad $\hat{Q}_h^t(s,a) = \hat{Q}_h^{\pi_t}(s,a; \overline{r}, \bar{p}) + \sum_{i=1}^{m} \lambda_{t,i} \hat{Q}_h^{\pi_t}(s,a; \overline{g}_{i}, \bar{p})$
    \STATE Compute truncated V-function:
        \STATE \quad$\hat{V}_{\overline{g}_i,h}^{\pi_t}(s) = \langle \hat{Q}_h^{\pi_t}(s, \cdot; \overline{g}_i, \overline{p}), \pi_h(\cdot \mid s) \rangle$\quad ($\forall i \in [m]$)
    \ENDFOR
\ENDFOR
\RETURN $\Big\{ \hat{Q}_h^{t}(s,a) \Big\}_{h,s,a}$ and $\Big\{\hat{V}_{\overline{g},h}^{\pi_t}\Big\}_{h}$
\end{algorithmic}
\end{algorithm}

\section{Feasibility}\label{app:feasibility}
This Appendix studies the feasibility of Projection \eqref{eq:projection}. 

\paragraph{Feasibility under pessimistic thresholds.} We first prove the feasibility of Projection \eqref{eq:projection} under pessimistic scenarios.

\begin{lemma}\label{lem:feasibility-p}
    Given a confidence parameter \(\delta \in (0, 1)\), Algorithm \ref{alg:spot} ensures that $\Pi_{\mathcal{C}}(\lambda_t,\overline{G}_{t-1},q^{\pi_t},\overline{\alpha}^t)$ is feasible at every episode \(t \in [T]\) with probability at least \(1 - \delta\).
\end{lemma}
\begin{proof}
To prove the lemma we show that under the success event $\mathcal{E}$, which holds with probability at least \(1 - \delta\), Projection \eqref{eq:projection} admits a feasible solution $q^{\pi^0}$. Precisely, under the success event $\mathcal{E}$, we have, for any feasible solution \(q^{\pi^0}\) under Assumptions \ref{asp-1} and \ref{asp-2} and any episode \(t \in [T]\),
\[
\overline{G}^\top q^{\pi^0} \succeq G_t^\top q^{\pi^0} \succeq \alpha+\zeta_t = \overline{\alpha} \succeq \alpha,
\]
\noindent
where the first inequality holds by the optimism of the constraints. The second and third inequalities show that if \(q^{\pi^0}\) satisfies the constraints with respect to the true mean constraint matrix, it also satisfies the pessimistic thresholds. Thus, the feasible solutions are all available at every episode. 
\end{proof}

\paragraph{Feasibility of optimistic policies.} We then prove the feasibility of Projection \eqref{eq:projection} of optimistic policies.

\begin{lemma}
    Given a confidence parameter \(\delta \in (0, 1)\), Algorithm \ref{alg:spot} ensures that $\Pi_{\mathcal{C}}(\lambda_t,\overline{G}_{t-1},q^{\pi_t},\overline{\alpha}^t)$ is feasible at every episode \(t \in [T]\) with probability at least \(1 - \delta\).
\end{lemma}
\begin{proof}
To prove the lemma we show that under the success event $\mathcal{E}$, which holds with probability at least \(1 - \delta\), Projection \eqref{eq:projection} admits a feasible solution. Similar to the discussion of Lemma \ref{lem:feasibility-p},
\[
\overline{G}^\top q^{\pi^0} \succeq G_t^\top q^{\pi^0} \succeq \alpha \succeq \alpha-\zeta_t=\underline{\alpha},
\]
\noindent
where the last inequality holds by the definition of the optimistic thresholds. This concludes the proof of the lemma.
\end{proof}

\section{Optimality and Optimism of Policies}\label{app:optimality}

In this Appendix, we analyse the optimality and optimism of policies. To clearly present our theoretical results, we introduce concise notation from convex optimisation and probability theory.

The optimistic constraints valuation under either pessimistic or optimistic scenarios (i.e., with scenario-dependent threshold $\alpha_{\diamond}^t$) and true constraints valuation are denoted by
    \begin{align*}
        \tilde{y}_t &:= \alpha_{\diamond}^t-\hat{V}_{\overline{g}}^{\pi_t}(s_1), \\
        y_t &:= \alpha^t-V_g^{\pi_t}(s_1).
    \end{align*}
We also define the negative values of the optimistic, actual, and optimal rewards respectively:
    \begin{align*}
        \tilde{f}_t &:= -\hat{V}_{\overline{r}}^{\pi_t}(s_1), \\
        f_t &:= -V_r^{\pi_t}(s_1), \\
        f_{opt} &:= -V_r^{\pi^{\star}}(s_1).
    \end{align*}

We present two lemmas that collectively control the variability of stochastic thresholds (Lemma \ref{lem:4}) and the summation of key terms under the Growing-Window estimator. Both have the order of $\tilde{O}(\sqrt{T})$. 

\begin{lemma}[Martingale difference term bound] \label{lem:martingale}
    Let $\{\mathcal{F}_{h}^t\}_{(t,h)\in[T]\times[H]}$ be the filtration generated by the state-action sequences and threshold values observed up to time step $h$ in episode $t$. Suppose that for every $(t,h)$, $\zeta_h^t(s_h,a_h)$ satisfies
$$\left|\zeta_h^t(s_h,a_h)-\mathbb{E}[\zeta_h^t(s_h,a_h)]\mid \mathcal{F}_{h}^{t-1}\right|\leq 2.$$
Then, with probability $1-\delta$, 
\[
\left|\sum_{t=1}^T\sum_{h=1}^H\left(\zeta_h^t(s_h,a_h)-\mathbb{E}\left[\zeta_h^t(s_h,a_h)\mid\mathcal{F}_{h}^{t-1}\right]\right)\right| \leq \tilde{\mathcal{O}}\left(\sqrt{TH\text{ln}\frac{1}{\delta}}\right).
\]
\end{lemma}

\begin{proof}
    We define a total order on $(t,h)$ by $$k(t,h)=(t-1)H+h$$
\noindent
for $t\in [T]$ and $h\in [H]$. For each $k \in [TH]$, set
$$D_k:=\zeta_k(s_h,a_h)-\mathbb{E}[\zeta_k(s_h,a_h)\mid\mathcal{F}_{h}^{t-1}].$$
\noindent
which holds that $\mathcal{F}_h^t \subset \mathcal{F}_{h'}^{t'}$ for any $t \leq t'$.

For any $k(t,h)\in [T]\times[H]$,
$$\mathbb{E}[D_k\mid\mathcal{F}_{h}^{t-1}]=0,$$
which means that $\{D_k\}_k$ is a martingale difference sequence with respect to the filtration $\{\mathcal{F}_h^{t-1}\}$ and satisfies $|D_k|\leq 2$. By the Azuma-Hoeffding inequality, for any $\epsilon >0$,
$$\mathbb{P}\left(\sum_{k=1}^{T\cdot H}D_k>\epsilon\right)\leq 2\text{exp}\left(-\frac{\epsilon^2}{2TH\cdot (2)^2}\right)=2\text{exp}\left(-\frac{\epsilon^2}{8TH}\right).$$
Setting the right-hand side equal to $\delta$ yields $\epsilon=2\sqrt{2TH\text{ln}\frac{2}{\delta}}$. Then, with probability $1-\delta$, we have
\[
\left|\sum_{t=1}^T\sum_{h=1}^H\zeta_h^t(s_h,a_h)-\mathbb{E}\left[\zeta_h^t(s_h,a_h)\mid\mathcal{F}_{h}^{t-1}\right]\right|=\left|\sum_{k=1}^{T\cdot  H}D_k\right|\leq 2\sqrt{2TH\text{ln}\frac{2}{\delta}}\lesssim\sqrt{TH\text{ln}\frac{1}{\delta}}.
\]
The proof is completed.
\end{proof}

\begin{lemma}\label{lem:gw-sum}
Assume a strictly increasing window-size sequence $\{W_t\}$ satisfying $W_t=\lfloor\gamma t\rfloor$ at episode $t$ with some constant $\gamma \in (0,1]$. For all $s, a, h, t \in [T]$, if
\[
N_h^{(W_t),t-1}(s,a) > \frac{1}{2} \sum_{j=\max\{1,t-W_t+1\}}^{t-1} q_h^{\pi_t}(s, a; p) - H \ln \frac{SAH}{\delta'},
\]
then
\[
\sum_{t=1}^{T} \sum_{h=1}^{H} \mathbb{E} \left[ \sqrt{\frac{1}{\max\{N_h^{(W_t),t-1}(s,a), 1\}}} \mid \mathcal{F}_{t-1} \right] 
\leq \tilde{\mathcal{O}}(\sqrt{SAH^2T/\gamma} + SAH).
\]
\end{lemma}
\begin{proof}
    We first define $w_{h}^t(s,a)=\mathbb{P}(s_h^t=s,a_h^t=a|\mathcal{F}_{t-1})$. Then, we have
    \begin{align*}\sum_{t=1}^{T} \sum_{h=1}^{H} \mathbb{E} \left[ \sqrt{\frac{1}{\max\{{N_h^{(W_t),t-1}(s,a) , 1\}}}} \mid \mathcal{F}_{t-1} \right]&=\sum_{t,h,s,a}w_{h}^t(s,a)\sqrt{\frac{1}{\max\{N_h^{(W_t),t-1}(s,a),1\}}}\\
    & \leq \sum_{t,h}\sum_{s,a\in L_t}w_h^t\sqrt{\frac{1}{N_h^{(W_t),t-1}(s,a)}}+\sum_{t,h}\sum_{s,a\in L_t^c}w_h^t\\
    &\leq \sum_{t,h}\sum_{s,a\in L_t}w_h^t\sqrt{\frac{1}{N_h^{(W_t),t-1}(s,a)}}+SAH,
    \end{align*}
    where we set $L_t=\left\{N_h^{(W_t),t-1}(s,a)\ge 1\right\}$ and thus set $L_t^c=\left\{N_h^{(W_t),t-1}(s,a)=0\right\}$. Consider the following term,
\begin{align*}  
\sum_{t,h}\sum_{s,a\in L_t}w_h^t\sqrt{\frac{1}{N_h^{(W_t),t-1}(s,a)}}&\leq \sqrt{\sum_{t,h}\sum_{s,a\in L_t}w_{h}^t}\sqrt{\sum_{t,h}\sum_{s,a\in L_t}\frac{w_{h}^t}{N_h^{(W_t),t-1}(s,a)}}\\
&\leq \sqrt{\sum_{t,h,s,a}w_{h}^t}\sqrt{\sum_{t,h}\sum_{s,a\in L_t}\frac{w_{h}^t}{N_h^{(W_t),t-1}(s,a)}}\\
&\leq \sqrt{TH}\sqrt{\sum_{t,h,s,a}\frac{w_{h}^t}{N_h^{(W_t),t-1}(s,a)}}\leq \tilde{O}(\sqrt{SAH^2T/\gamma}).
\end{align*}
The first relation holds due to Cauchy–Schwarz inequality. 

Combining all the results concludes the proof.
\end{proof}

\subsection{Optimality under Pessimistic Thresholds}\label{app:optimality-pes}
In this section, we focus on the case where thresholds are set pessimistically. The following lemma provides a performance bound (Lemma \ref{lem:pes-opt}) showing that despite the additional caution, the policies remain near-optimal under these pessimistic thresholds.
\begin{customlem}{3}[Policy optimality under pessimistic thresholds]\label{lem:optimality-pes} 
    Conditioned on the success event $\mathcal{E}$, for any episode $t \in [T']$ and parameter $\gamma \in (0,1]$, if we use pessimistic thresholds, then
\[\sum_{t=1}^{T'} \tilde{f}_t + \lambda_t^{\top} \tilde{y}_t - f_{\pi^{\star}} - \lambda_t^{\top} y_{\pi^{\star}} 
\leq \tilde{\mathcal{O}}\left(\sqrt{H^4(1 + m\rho)^2 T}+\sqrt{m^2\rho^2SAH^2T/\gamma}\right).\]
\end{customlem}
\begin{proof}
For each episode $t$, according to the definitions, we obtain:
\begin{equation*}
    \tilde{f}_t + \lambda_t^{\top} \tilde{y}_t - f_{\pi^{\star}} - \lambda_t^{\top} y_{\pi^{\star}} =V_r^{\pi^{\star}}(s_1)-\hat{V}_{\overline{r}}^{\pi_t}(s_1)+\sum_{i=1}^m\lambda_{t,i}\left(\overline{\alpha}_{i}^t-\alpha_i\right)+\sum_{i=1}^m\lambda_{t,i}\left(V_{g_i}^{\pi^{\star}}(s_1)-\hat{V}_{\overline{g}_i}^{\pi_t}(s_1)\right).
\end{equation*}
We sum it over episode $T'$ and decompose it as 
\begin{align*}
    &\sum_{t=1}^{T'} \tilde{f}_t + \lambda_t^{\top} \tilde{y}_t - f_{\pi^{\star}} - \lambda_t^{\top} y_{\pi^{\star}}\\
    =& \sum_{t=1}^{T'} \left[V_{r+\lambda g}^{\pi^{\star}}(s_1)-V_{\overline{r}+\lambda \overline{g}}^{\pi_t}(s_1) +\sum_{i=1}^{m}\lambda_{t,i}(\overline{\alpha}_{i}^t-\alpha_i)\right]\\
    =&\sum_{t=1}^T \sum_{i=1}^m\lambda_{t,i} \left[\sum_{h=1}^H\overline{\alpha}_{i}^t(s_h, a_h)-\alpha_{i}(s_h, a_h)\right]\tag{$i$}\\
    +& \sum_{t=1}^{T} \sum_{h=1}^{H} \mathbb{E} \Big[ \langle \hat{Q}_h^t(s_h, \cdot),\pi_h^{\star}(\cdot \mid s_h)-\pi_h^t(\cdot \mid s_h) \rangle \mid s_1, \pi^{\star}, p \Big]\tag{$ii$}\\
    +& \sum_{t=1}^{T} \sum_{h=1}^{H} \mathbb{E} \left[ -\hat{Q}_h^t(s_h, a_h) + r_h(s_h, a_h) + \sum_{i=1}^{m} \lambda_t g_{i,h}(s_h, a_h) + p_h(\cdot \mid s_h, a_h) \hat{V}_{h+1}^t \mid s_1, \pi^{\star}, p \right]\tag{$iii$}
\end{align*}
where composite $Q$-value function is $\hat{Q}_h^t(s,a) = \hat{Q}_h^{\pi_t}(s,a; \overline{r}_t, \bar{p}_{\pi_t}) + \sum_{i=1}^{m} \lambda_{t,i} \hat{Q}_h^{\pi_t}(s,a; \overline{g}_{t,i}, \bar{p}_{t})$ for each episode $t$ and step $h$, composite $V$-value function is defined by $\hat{V}_h^t(s) := \langle \hat{Q}_h^t(s, \cdot), \pi_h^t \rangle$ and the second equality holds because of Lemma \ref{lem:value difference}.

We bound these terms separately:
\paragraph{First term (threshold estimation error)} Under the pessimistic thresholds, for all $s,a,h,i$, it holds that 
$$\alpha_{i}(s_h, a_h)\leq \overline{\alpha}_{i}^t(s_h, a_h),$$
so that
$$(i)=\sum_{t=1}^T \sum_{i=1}^m\lambda_{t,i} \left[\sum_{h=1}^H\overline{\alpha}_{i}^t(s_h, a_h)-\alpha_{i}(s_h, a_h)\right]$$
\noindent
can be upper bounded as follows:
\begin{align*}
    \sum_{i=1}^m\lambda_{t,i} \left[\sum_{h=1}^H\overline{\alpha}_{i}^t(s_h, a_h)-\alpha_{i}(s_h, a_h)\right] &\leq \sum_{i=1}^m\lambda_{t,i} \left[\sum_{h=1}^H\left|\overline{\alpha}_{i}^t(s_h, a_h)-\alpha_{i}(s_h, a_h)\right|\right]\\
    &\leq m\max_{i\in [m]}\lambda_{t,i}\left[\sum_{h=1}^H\left|\overline{\alpha}_{i}^t(s_h, a_h)-\alpha_{i}(s_h, a_h)\right|\right]\\
    &\leq m\rho\sum_{h=1}^H 2\zeta_h^{t}(s_h,a_h).
\end{align*}

By applying Lemmas \ref{lem:martingale} and \ref{lem:gw-sum}, and summing over $t$ and $h$, we further have
\begin{align}\label{eq:pes-1}
    &\sum_{t=1}^T \sum_{i=1}^m\lambda_{t,i} \left[\sum_{h=1}^H\overline{\alpha}_{i}^t(s_h, a_h)-\alpha_{i}(s_h, a_h)\right]
\leq 2m\rho \sum_{t=1}^T\sum_{h=1}^H \zeta_h^{t}(s_h,a_h)\nonumber\\
&\lesssim m\rho\left[\sum_{t=1}^T\sum_{h=1}^H\mathbb{E}\left[\zeta_h^{t}(s_h,a_h)\mid s_1,\pi,p)\right]+\sqrt{TH^3\text{ln}\frac{1}{\delta}}\right]\nonumber\\
&\lesssim m\rho\left[\sqrt{SAH^2T/\gamma}+SAH+\sqrt{TH\text{ln}\frac{1}{\delta}}\right]
\end{align}

\paragraph{Second term (policy difference)}
For the second term, by Lemma \ref{lem:OMD} (with $\pi=\pi^{\star}$), we obtain:
\begin{equation}\label{eq:pes-2}
    (ii) = \sum_{t=1}^{T'} \sum_{h=1}^{H} 
\mathbb{E} \left[ \Big\langle \hat{Q}_h^t(s_h, \cdot), \pi_h^{\star}(\cdot \mid s_h)-\pi_h^t(\cdot \mid s_h) \Big\rangle
\mid s_1, \pi^{\star}, p \right]
\lesssim \sqrt{H^4 (1 + m \rho)^2 T}.
\end{equation}

\paragraph{Third term (Bellman residual)}
Moreover, Lemma \ref{lem:policy estimation opt} shows that for all $s, a, h, t$, 
\begin{equation*}
    r_h(s,a) + \sum_{i=1}^{m} \lambda_t g_{i,h}(s,a) + p_h(\cdot \mid s,a) \hat{V}_{h+1}^t -\hat{Q}_h^t(s,a)\leq 0,
\end{equation*}
so ($iii$) satisfies
\begin{equation}\label{eq:pes-3}
    (iii) \leq 0. 
\end{equation}
Combining \eqref{eq:pes-1}, \eqref{eq:pes-2}, and \eqref{eq:pes-3}, we conclude the result:
\begin{equation*}
    \sum_{t=1}^{T'} \tilde{f}_t + \lambda_t^{\top} \tilde{y}_t - f_{\pi^{\star}} - \lambda_t^{\top} y_{\pi^{\star}} 
\lesssim \sqrt{H^4(1 + m\rho)^2 T}+\sqrt{m^2\rho^2SAH^2T/\gamma}+m\rho SAH+\sqrt{m^2\rho^2 TH\text{ln}\frac{1}{\delta}}.
\end{equation*}
\end{proof}

\subsection{Optimality of Optimistic Policies}\label{app:optimality-opt}
We now turn to the scenario in which thresholds are chosen optimistically. Although such thresholds may facilitate exploration and potentially higher returns, it is essential to establish theoretical guarantees (Lemma \ref{lem:opt-opt}) that they do not lead to excessively optimistic estimates. The next lemma shows that even with looser constraints, the performance remains near-optimal under standard assumptions.
\begin{customlem}{5}[Policy optimality with optimistic thresholds] \label{lem:optimality-opt} 
    Conditioned on the success event $\mathcal{E}$, for any episode $t \in [T']$, if we adopt the optimistic thresholds, then
    \begin{equation*}
     \sum_{t=1}^{T'} \tilde{f}_t + \lambda_t^{\top} \tilde{y}_t - f_{\pi^{\star}} - \lambda_t^{\top} y_{\pi^{\star}} \leq \tilde{\mathcal{O}}\left((1+m\rho')H^2\sqrt{T}\right).   
    \end{equation*}
where $\rho'$ denotes the Slater gap under optimistic thresholds.
\end{customlem}

\begin{proof}
Following the similar initial steps in the proof of Lemma \ref{lem:optimality-pes}, we decompose $\sum_{t=1}^{T'} \tilde{f}_t + \lambda_t^{\top} \tilde{y}_t - f_{\pi^{\star}} - \lambda_t^{\top} y_{\pi^{\star}}$ into three terms (I), (II) and (III). The second and third terms (II) and (III) are identical to terms ($ii$) and ($iii$). As previously proved, we have
\begin{equation}\label{eq:opt-1}
    (\text{II})\lesssim (1+m\rho')H^2\sqrt{T} \quad \text{and} \quad (\text{III}) \leq 0
\end{equation}
However, term (I) exhibits a significant difference under optimistic thresholds compared to the pessimistic scenario as follows:
\begin{equation*}
    (\text{I}) = \sum_{t=1}^T \sum_{i=1}^m\lambda_{t,i} \left[\sum_{h=1}^H\underline{\alpha}_{i}^t(s_h, a_h)-\alpha_{i}(s_h, a_h)\right].
\end{equation*}
Specifically, under optimistic thresholds, it holds for all $s,a,h,t,i$ that
$$\alpha_{i}(s_h, a_h)\geq \underline{\alpha}_{i}^t(s_h, a_h),$$
\noindent
which implies 
\begin{equation}\label{eq:opt-4}
    (iii) \leq 0.
\end{equation}

Combining \eqref{eq:opt-1} and \eqref{eq:opt-4}, we immediately obtain
\[
\sum_{t=1}^{T'} \tilde{f}_t + \lambda_t^{\top} \tilde{y}_t - f_{\pi^{\star}} - \lambda_t^{\top} y_{\pi^{\star}} 
\lesssim \tilde{\mathcal{O}}\left((1+m\rho')H^2\sqrt{T}\right).
\]
\end{proof}

\section{Regret Analysis}
\label{app:main theorem}
In this section, we present omitted proofs for the regret analysis in Section \ref{sec:thms}.

\subsection{Proof of Theorem \ref{thm:pes}}
We start by providing the proof under pessimistic thresholds.
\begin{customthm}{2}[Bounds for reward regret and constraint violation of pessimistic policies]\label{thm:pes-proof}
    Let $S$ be the number of states, $A$ be the number of actions, $H$ be the horizon length, $m$ be the number of constraints, $\rho$ be the Slater gap under pessimistic thresholds, $\gamma$ be the Growing-Window parameter, $T$ be the number of episodes and $\mathcal{N}$ denote the maximum number of non-zero
transition probabilities across all state-action pairs. Then with probability $1-\delta$, Algorithm \ref{alg:spot} under pessimistic thresholds obtains the following regret bounds:
    \begin{align*}
    \mathcal{R}_T(r) &\leq \tilde{\mathcal{O}} \left( \sqrt{S \mathcal{N} H^4 T} + \sqrt{H^4(1 + m\rho)^2 T}+\sqrt{m^2\rho^2SAH^2T/\gamma} + (\sqrt{\mathcal{N}} + H) H^2 S A \right),\\
    \mathcal{R}_T(g) &\leq \tilde{\mathcal{O}} \left( \frac{1}{\rho} \left( \sqrt{ S \mathcal{N} H^4 T} + (\sqrt{\mathcal{N}} + H)  H^2 S A \right) + m \sqrt{SAH^4 T/\gamma}\right),
\end{align*}
where $\tilde{\mathcal{O}}$ hides polylogarithmic factors in $(S,A,H,T,1/\delta)$.
\end{customthm}
\begin{proof} We begin by leveraging the properties of dual optimisation from Lemmas \ref{lem:dual optimism} and Lemma \ref{lem:update rule}, rearranging and obtaining that for any $\lambda$ with $0\leq \lambda\leq \rho 1$,

\begin{align*}
    \mathcal{L}(\lambda)=&\sum_{t=1}^{T'} (f_t - f_{\text{opt}}) + \sum_{t=1}^{T'} y_t^{\top} \lambda\\
    \leq& \sum_{t=1}^{T'} (y_t - \tilde{y}_t)^{\top} \lambda+\sum_{t=1}^{T'} \left(\tilde{f}_t + \lambda_t^{\top} \tilde{y}_t - f_{\pi^{\star}} - \lambda_t^{\top} y_{\pi^{\star}}\right)+\sum_{t=1}^{T'} (f_t - \tilde{f}_t)+ \frac{1}{2\eta_\lambda} \sum_{t=1}^{T'} \|\tilde{y}_t\|^2  +\frac{\eta_\lambda}{2} \|\lambda\|_2^2.
\end{align*}

Under the pessimistic thresholds, we have $y_t\leq \tilde{y}_t$ at episode $t$, thus 
\begin{equation*}
    \sum_{t=1}^{T'} (y_{t} - \tilde{y}_{t})\lambda\leq 0,
\end{equation*}
leading us to simplify the above inequality as
\begin{equation}\label{eq:decomposition2}
    \mathcal{L}(\lambda)\leq  
    \underbrace{\sum_{t=1}^{T'}\left( \tilde{f}_t + \lambda_t^{\top} \tilde{y}_t - f_{\pi^{\star}} - \lambda_t^{\top} y_{\pi^{\star}}\right)}_{(a)}
    +\underbrace{\sum_{t=1}^{T'} (f_t - \tilde{f}_t)}_{(b)}
    +\underbrace{\frac{1}{2\eta_\lambda} \sum_{t=1}^{T'} \|\tilde{y}_t\|^2}_{(c)}
    +\frac{\eta_\lambda}{2} \|\lambda\|_2^2.
\end{equation}
We proceed by individually bounding each term:
\paragraph{First term (policy optimality gap)} Under the pessimistic thresholds, it holds that
$\lambda_{t,i}(\overline{\alpha}_{i}^t-\alpha_i)\geq 0$ for any episode $t$ and constraint $i$. Thus by Lemma \ref{lem:optimality-pes}, we derive
\begin{equation*}
\sum_{t=1}^{T'} \left(\tilde{f}_t + \lambda_t^{\top} \tilde{y}_t - f_{\pi^{\star}} - \lambda_t^{\top} y_{\pi^{\star}}\right) 
\lesssim \sqrt{H^4(1 + m\rho)^2 T}+\sqrt{m^2\rho^2SAH^2T/\gamma}.
\end{equation*}

\paragraph{Second term (reward estimation error)} By Lemma \ref{lem:policy errors-truncated}, we get
\begin{align*}
    \left| \sum_{t=1}^{T'} (f_t - \tilde{f}_t) \right| = \left| \sum_{t=1}^{T'} \left(\hat{V}_{\overline{r}}^{\pi_t}(s_1)-V_r^{\pi_t}(s_1)\right) \right|
    \leq \tilde{\mathcal{O}} \left( \sqrt{S \mathcal{N} H^4 T} + (\sqrt{\mathcal{N}} + H) H^2 S A \right).
\end{align*}

\paragraph{Third term (dual update bound)} Due to $|\tilde{y}_{i,t}| \leq H$ for constraint $i\in [m]$, we have
\[
\frac{\sum_{t=1}^{T'} \|\tilde{y}_t\|^2}{2\eta_\lambda} \leq \frac{m H^2  T}{2\eta_\lambda}.
\]

Substituting these results back into \eqref{eq:decomposition2} and choosing the learning rate $\eta_\lambda = \sqrt{\frac{mH^2T}{\rho^2}}$, we achieve the intermediate bound, for any $0 \leq \lambda \leq \rho \mathbf{1}$,
\begin{align}\label{eq:pes-5}
    \sum_{t=1}^{T'} (f_t - f_{\text{opt}}) + \sum_{t=1}^{T'} y_t^{\top} \lambda\nonumber
    \lesssim& (\rho + \frac{\|\lambda\|_2^2}{\rho}) \sqrt{mH^2  T} + \left( \sqrt{S \mathcal{N} H^4 T} + (\sqrt{\mathcal{N}} + H) H^2 S A \right)\nonumber\\
    &+ \sqrt{H^4(1 + m\rho)^2 T}+\sqrt{m^2\rho^2SAH^2T/\gamma}.
\end{align}
\noindent

\paragraph{Reward regret} Setting $\bf\lambda = 0$ in \eqref{eq:pes-5}, we directly obtain the regret bound for rewards:
\begin{equation*}
    \mathcal{R}_T(r)
    \leq \tilde{\mathcal{O}} \left( \sqrt{S \mathcal{N} H^4 T} + (\sqrt{\mathcal{N}} + H) H^2 S A  + \sqrt{H^4(1 + m\rho)^2 T}+\sqrt{m^2\rho^2SAH^2T/\gamma}\right).
\end{equation*}
\paragraph{Constraint violation} For constraint $i \in [m]$, set the vector $\lambda_i=\rho e_i$ if $[\sum_{t=1}^{T'}y_{t,i}]_+>0$ and $0$ otherwise in \eqref{eq:pes-5}, where $e_i$ denotes the standard basis vector and the Slater gap $\rho$ is defined in Assumption \ref{asp-1}. This also guarantees $\|\lambda_i\|_2^2\leq \rho^2$. 
We then have
\begin{align}\label{eq:pes-7}
    &\sum_{t=1}^{T'} (f_t - f_{\text{opt}}) + \rho \left[ \sum_{t=1}^{T'} y_{i,t} \right]_+\nonumber\\
    \lesssim&
\left( \sqrt{ S \mathcal{N} H^4 T} + (\sqrt{\mathcal{N}} + H)H^2 S A \right) + \sqrt{H^4(1 + m\rho)^2 T}+\sqrt{m^2\rho^2SAH^2T/\gamma}:=\theta(T).
\end{align}
Due to the linearity of both the reward and constraints, combined with the convexity of the occupancy-measure set, we leverage Lemmas \ref{lem:F.2} and \ref{Lem:F.3} and then derive the following result for any $T' \in [T]$:
\begin{equation*}
    \max_{i \in [m]} \left[ \sum_{t=1}^{T'} y_t \right] \leq \left\| \left[\sum_{t=1}^{T'} y_t \right]_+\right\|_\infty \leq \frac{\theta(T)}{\rho}.
\end{equation*}
Therefore, we conclude 
\begin{equation*}
    \mathcal{R}_T(g) \leq \tilde{\mathcal{O}} \left( \frac{1}{\rho} \left( \sqrt{ S \mathcal{N} H^4 T} + (\sqrt{\mathcal{N}} + H)  H^2 S A \right) + m \sqrt{SAH^4 T/\gamma}\right).
\end{equation*}
\end{proof}

\subsection{Proof of Theorem \ref{thm:opt}}
Next, we prove the bounds of reward regret and constraint violation for the optimistic thresholds carefully. 
\begin{customthm}{3}[Bounds for reward regret and constraint violation with optimistic thresholds]
    Under the same conditions of Theorem \ref{thm:pes}, with probability $1-\delta$, Algorithm \ref{alg:spot} under optimistic thresholds holds the following regret bounds:
\begin{align*}
\mathcal{R}_T(r) \leq& \tilde{\mathcal{O}} \left( \sqrt{S \mathcal{N} H^4 T} + \sqrt{H^4 (1 + m\rho')^2 T} + (\sqrt{\mathcal{N}} + H) H^2 S A \right),\\
\mathcal{R}_T(g) \leq& \tilde{\mathcal{O}} \left( \left(1 + \frac{1}{\rho'} \right) \left( \sqrt{m S \mathcal{N} H^4 T} + (\sqrt{\mathcal{N}} + H) \sqrt{m} H^2 S A\right) + m \sqrt{H^4 T}\right)\\
&+\tilde{\mathcal{O}}\left(\sqrt{mSAH^2T/\gamma}+\sqrt{m}SAH\right).
\end{align*}
where $\rho'$ denotes the Slater gap under optimistic thresholds.
\end{customthm}

\begin{proof} We follow the similar initial decomposition as in the pessimistic scenario (Theorem \ref{thm:pes}), leading to the following inequality:
\begin{align}\label{eq:opt-L}
    \mathcal{L}(\lambda)=&\sum_{t=1}^{T'} (f_t - f_{\text{opt}}) + \sum_{t=1}^{T'} y_t^{\top} \lambda\nonumber\\
    \leq& \sum_{t=1}^{T'} (y_t - \tilde{y}_t)^{\top} \lambda+\sum_{t=1}^{T'} \left(\tilde{f}_t + \lambda_t^{\top} \tilde{y}_t - f_{\pi^{\star}} - \lambda_t^{\top} y_{\pi^{\star}}\right)+\sum_{t=1}^{T'} (f_t - \tilde{f}_t)+ \frac{1}{2\eta_\lambda} \sum_{t=1}^{T'} \|\tilde{y}_t\|^2  +\frac{\eta_\lambda}{2} \|\lambda\|_2^2\nonumber\\
    \leq&\underbrace{\sqrt{\sum_{i=1}^{m} \left( \sum_{t=1}^{T'} (y_{t,i} - \tilde{y}_{t,i}) \right)^2} \|\lambda\|_2}_{(a')}+\underbrace{\sum_{t=1}^{T'} \left(\tilde{f}_t + \lambda_t^{\top} \tilde{y}_t - f_{\pi^{\star}} - \lambda_t^{\top} y_{\pi^{\star}}\right)}_{(b')}\nonumber\\
    &+\underbrace{\sum_{t=1}^{T'} (f_t - \tilde{f}_t)}_{(c')}+ \underbrace{\frac{1}{2\eta_\lambda} \sum_{t=1}^{T'} \|\tilde{y}_t\|^2}_{(d')}  +\underbrace{\frac{\eta_\lambda}{2} \|\lambda\|_2^2}_{(e')},
\end{align}
where the second relation holds due to Cauchy–Schwarz inequality.

We bound each term individually, and the principal difference from the pessimistic scenario lies in controlling terms (a') and (b'):
\paragraph{First term (constraint and threshold estimation error)} Unlike pessimistic thresholds, for the optimistic thresholds, it holds that
$\lambda_{t,i}(\alpha_i-\underline{\alpha}_{i}^t)\geq 0$ for any constraint $i$ and episode $t$.
We then obtain
\begin{align}\label{eq:opt-first term}
  \left| \sum_{t=1}^{T'} (y_{t,i} - \tilde{y}_{t,i}) \right|&= \left| \sum_{t=1}^{T'} \left(\hat{V}_{\overline{g}_i}^{\pi_t}(s_1)-V_{g_i}^{\pi_t}(s_1)\right) +\sum_{t=1}^{T'}\left(\alpha_i-\underline{\alpha}_{i}^t\right)\right|\nonumber\\
  &\leq \left| \sum_{t=1}^{T'} \left(\hat{V}_{\overline{g}_i}^{\pi_t}(s_1)-V_{g_i}^{\pi_t}(s_1)\right) \right|+\left|\sum_{t=1}^{T'}\left(\alpha_i-\underline{\alpha}_{i}^t\right)\right|.  
\end{align}
According to Equation \eqref{eq:opt-first term}, we have
\begin{align*}
   \sqrt{\sum_{i=1}^{m} \left( \sum_{t=1}^{T'} (y_{t,i} - \tilde{y}_{t,i}) \right)^2 } \leq& \sqrt{m}\max_{i\in [m]}\left|\sum_{t=1}^{T'} (y_{t,i} - \tilde{y}_{t,i})\right|\\
   \leq& \sqrt{m}\max_{i\in [m]}\left[\left| \sum_{t=1}^{T'} \left(\hat{V}_{\overline{g}_i}^{\pi_t}(s_1)-V_{g_i}^{\pi_t}(s_1)\right) \right|+\left|\sum_{t=1}^{T'}(\alpha_i-\underline{\alpha}_{i}^t)\right|\right]\\
   \leq& \tilde{\mathcal{O}} \left( \sqrt{m S \mathcal{N} H^4 T} + (\sqrt{\mathcal{N}} + H +1) \sqrt{m} H^2 SA+\sqrt{mSAH^2T/\gamma}\right).
\end{align*}
Thus
\begin{equation*}
    (a')\leq\tilde{\mathcal{O}} \left( \sqrt{m S \mathcal{N} H^4 T} + (\sqrt{\mathcal{N}} + H +1) \sqrt{m} H^2 SA+\sqrt{mSAH^2T/\gamma}\right)\|\lambda\|_2.
\end{equation*}
\paragraph{Second term (policy optimality gap)} 
Under optimistic thresholds, by Lemma \ref{lem:optimality-opt}, we have
\begin{equation*}
    (b')=\sum_{t=1}^{T'} \left(\tilde{f}_t + \lambda_t^{\top} \tilde{y}_t - f_{\pi^{\star}} - \lambda_t^{\top} y_{\pi^{\star}}\right) \lesssim (1+m\rho')H^2\sqrt{T}.
\end{equation*}
\paragraph{Third and fourth terms} They are the same as in the pessimistic setting. We restate here that
\begin{equation*}
    (c')=\frac{\sum_{t=1}^{T'} \|\tilde{y}_t\|^2}{2\eta_\lambda}  \leq \frac{mH^2T}{2\eta_\lambda}.
\end{equation*}
\noindent
Under the success event $\mathcal{E}$, we have
\begin{equation*}
    (d')\leq \left| \sum_{t=1}^{T'} (f_t - \tilde{f}_t) \right| \leq \tilde{\mathcal{O}} \left( \sqrt{S \mathcal{N} H^4 T} + (\sqrt{\mathcal{N}} + H) H^2 S A \right).
\end{equation*}

Incorporating all pieces and then choosing $\eta_\lambda = \sqrt{\frac{mH^2T}{\rho'^2}}$, we derive an overall bound 
\begin{align}\label{eq:opt-8}
    \sum_{t=1}^{T'} (f_t - f_{\text{opt}}) + \sum_{t=1}^{T'} y_t^{\top} \lambda\nonumber
    \lesssim&  \left( \sqrt{m S \mathcal{N} H^4 T} + (\sqrt{\mathcal{N}} + H) \sqrt{m} H^2 S A\right)\|\lambda\|_2\nonumber\\
    &+\left(\sqrt{mSAH^2T/\gamma}+\sqrt{m}SAH\right) \|\lambda\|_2+(1+m\rho')H^2\sqrt{T}\nonumber\\
    &+ \left( \sqrt{S \mathcal{N} H^4 T} + (\sqrt{\mathcal{N}} + H) H^2 S A \right)+\left(\rho' + \frac{\|\lambda\|_2^2}{\rho'}\right) \sqrt{mH^2T}.
\end{align}
\paragraph{Reward regret} Taking $\bf\lambda = 0$ in \eqref{eq:opt-8}, we get the reward regret
\begin{equation*}
    \mathcal{R}_T(r)\leq \tilde{\mathcal{O}} \left( \sqrt{S \mathcal{N} H^4 T} + (\sqrt{\mathcal{N}} + H) H^2 S A  + (1+m\rho')H^2\sqrt{T} \right).
\end{equation*}
\paragraph{Constraint violation} For constraint violation, set $\lambda_i =\rho'e_i\textbf{1}_{\{\left[\sum_{t=1}^{T'} y_{i,t}\right]_+ \neq 0\}}$ for constraint $i\in [m]$.
We can obtain
\begin{align}\label{eq:opt-10}
    \sum_{t=1}^{T'} (f_t - f_{\text{opt}}) + \rho' \left[ \sum_{t=1}^{T'} y_{i,t} \right]_+\nonumber
    \leq& 
(1 + \rho') \left( \sqrt{m S \mathcal{N} H^4 T} + (\sqrt{\mathcal{N}} + H)\sqrt{m} H^2 S A\right)\nonumber\\
&+\rho'\left(\sqrt{mSAH^2T/\gamma}+\sqrt{m} SAH\right)+ (1+m\rho')H^2\sqrt{T}\nonumber\\
&:=\theta'(T).\nonumber
\end{align}
\end{proof}
 By Lemmas \ref{lem:F.2} and \ref{Lem:F.3}, we  have the following result for any $T' \in [T]$:
$\max_{i \in [m]} \left[ \sum_{t=1}^{T'} y_t \right] \leq \frac{\theta'(T)}{\rho'}$. Therefore, we conclude the bound for constraint violation
\begin{align*}
    \mathcal{R}_T(g) \leq& \tilde{\mathcal{O}} \left( \left(1 + \frac{1}{\rho'} \right) \left( \sqrt{m S \mathcal{N} H^4 T} + (\sqrt{\mathcal{N}} + H) \sqrt{m} H^2 S A\right) + m \sqrt{H^4 T}\right)\\
&+\tilde{\mathcal{O}}\left(\sqrt{mSAH^2T/\gamma}+\sqrt{m}SAH\right).
\end{align*}

\section{Constrained Convex Optimisation}

We consider the following optimisation problem:
\begin{equation}\label{eq:convex-1}
\begin{aligned}
f^{\star} \;=\; \min_{x} \; &f(x) \\
 \text{s.t. } &y_i(x) \le 0, \quad i \in [m], \\
 \quad &x \in X,
\end{aligned}
\end{equation}
accompanied by the assumption stated below.

\begin{assumption}[\cite{beck2017first}]
\label{assumption:convex}
In the context of \eqref{eq:convex-1}, we assume:
\begin{enumerate}
    \item $X \subset \mathbb{R}^n$ is convex.
    \item $f(\cdot)$ is convex.
    \item Each $y_i(\cdot)$ is concave for $i \in [m]$ and $y(\cdot):= (y_1(\cdot),\ldots,y_m(\cdot))^{\top}$.
    \item The problem in \eqref{eq:convex-1} has a finite optimal value $f^{\star}$, and there exists at least one $x \in X$ that attains this optimum, i.e., $X^{\star} \neq \emptyset$.
    \item There exists $\overline{x}\in X$ such that $y(\overline{x})<0$.
    \item For every $\lambda \in \mathbb{R}_{\ge0}^m$, the minimisation problem $\min_{x\in X}[f(x)+\lambda^{\top}y(x)]$ admits an optimal solution.
\end{enumerate}
\end{assumption}

Under these conditions, we define the dual objective function as
\[
q(\lambda) \;=\; \min_{x \in X} \; \mathcal{L}(x, \lambda)
\;=\; \min_{x \in X} \Bigl[f(x) + \lambda^{\top} y(x)\Bigr],
\]
where $\mathcal{L}(\cdot,\cdot)$ is the Lagrangian associated with the problem in \eqref{eq:convex-1}. The corresponding dual problem is then
\begin{align*}
 q^{\star} \;=\; \max \; &q(\lambda).\\
 \text{s.t. }&\lambda \ge 0.
\end{align*}
Within this framework, we have the following result that establishes a connection between the primal and dual problem.

\begin{lemma}[\cite{beck2017first}]
\label{lem:convex}
Under Assumption \ref{assumption:convex}, the strong duality property holds; that is, 
$$f^{\star} = q^{\star}.$$ Moreover, an optimal solution to the dual problem exists, with the set of optimal solutions $\Lambda^{\star} \neq \emptyset$.
\end{lemma}

\section{Preparation Lemmas} 

In this Appendix, we present important preparation lemmas from prior work \cite{efroni2019tight, efroni2020exploration}.
\begin{lemma}[\cite{efroni2020exploration}]\label{lem:value difference}
Let $\pi,\pi'$ be two policies, and $\mathcal{M}=(\mathcal{S},\mathcal{A},r,p)$ and $\mathcal{M}'=(\mathcal{S},\mathcal{A},r',p')$ be two MDPs. Let $\widehat{Q}_h^\pi(s,a;r,p)$ be an approximation of the $Q$-function of policy $\pi$ on the MDP $\mathcal{M}$ for all $h,s,a$, and let $\widehat{V}_h^\pi(s;r,p)=\left\langle \widehat{Q}_h^\pi(s,\cdot;r,p),\pi_h(\cdot\mid s)\right\rangle$. Then,
\begin{align*}
    &V_{r'}^{\pi'}(s_1)-\widehat{V}_r^{\pi}(s_1)=\\
    &\quad\sum_{h=1}^{H}\mathbb{E}\left[\left\langle\widehat{Q}_h^\pi(s_h,\cdot),\pi_h(\cdot\mid s_h)-\pi'_h(\cdot\mid s_h)\right\rangle\mid s_1,\pi',p'\right]+\\
    &\quad\sum_{h=1}^{H}\mathbb{E}\left[-\widehat{Q}_h^\pi(s_h,a_h;r,p)+r'_h(s_h,a_h)+p'_h(\cdot\mid s_h,a_h)\widehat{V}_{h+1}^\pi(\cdot;r,p)\mid s_1,\pi',p'\right]
\end{align*}
where $V_1^{\pi'}(s;r',p')$ is the value function of $\pi'$ in the MDP $\mathcal{M}'$.
\end{lemma}
\begin{lemma}[\cite{efroni2019tight}]\label{lem:estimation-sum}
Assume that for all $s, a, h, t \in [T]$
\[
N_h^{t}(s, a) > \frac{1}{2} \sum_{j<t} q_h^{\pi_j}(s, a; p) - H \ln \left(\frac{SAH}{\delta'}\right),
\]
then
\[
\sum_{t=1}^{T} \sum_{h=1}^{H} \mathbb{E} \left[ \sqrt{\frac{1}{N_h^{t}(s_h^t, a_h^t) \vee 1}} \mid \mathcal{F}_{t-1} \right] 
\leq \tilde{\mathcal{O}}(\sqrt{SAH^2T} + SAH).
\]
\end{lemma}
\begin{lemma}[\cite{efroni2020exploration}]\label{lem:policy estimation opt} 
    Under the success event, for any $s, a, h, t$, the following bound holds:
\[
r_h(s,a) + \sum_{i=1}^{m} \lambda_t g_{i,h}(s,a) + p_h(\cdot \mid s,a) V_{h+1}^t -Q_h^t(s,a)\leq 0.
\]
\textit{where}
\[
Q_h^t(s,a) = Q_h^{\pi_t}(s,a; \overline{r}_t, \bar{p}_{t}) + \sum_{i=1}^{m} \lambda_{t,i} Q_h^{\pi_t}(s,a; \overline{g}_{t,i}, \bar{p}_{t}), 
\]

\[
V_h^t(s) = \langle Q_h^t(s, \cdot), \pi_h^t(\cdot \mid s) \rangle. 
\]
\end{lemma}
\noindent
See that $Q_h^{\pi_t}(s,a; \overline{r}_t, \bar{p}_{t-1})$, $Q_h^{\pi_t}(s,a; \overline{g}_{t,i}, \bar{p}_{t-1})$ are defined in the update rule of Algorithm \ref{alg:spot}.

\begin{lemma}[OMD term bound \cite{efroni2020exploration}] \label{lem:OMD} 
    Conditioned on the success event, we have that for any $\pi$
\[
\sum_{t=1}^{T} \sum_{h=1}^{H} 
\mathbb{E} \left[ \Big\langle Q_h^t(s_h, \cdot), \pi_h(\cdot \mid s_h)-\pi_h^t(\cdot \mid s_h) \Big\rangle
\mid s_1 = s, \pi, p \right]
\leq \sqrt{2 H^4 (1 + m \rho)^2 T \log A}.
\]
\end{lemma}
\begin{lemma}[Dual optimism \cite{efroni2020exploration}] \label{lem:dual optimism} 
    Conditioning on the success event, for any $t \in [T']$
\[
\tilde{f}_t - f_{\text{opt}} \leq -\lambda_t^{\top} \tilde{y}_t + \left( \tilde{f}_t + \lambda_t^{\top} \tilde{y}_t - f_{\pi^{\star}} - \lambda_t^{\top} y_{\pi^{\star}} \right).
\]
\end{lemma}

\begin{lemma}[On policy errors for truncated policy estimation \cite{efroni2020exploration}]
\label{lem:policy errors-truncated}
    Let $l_h(s,a), \overline{l}_h^t(s,a)$ be a cost function, and its optimistic cost. 
Let $p$ be the true transition dynamics of the MDP and $\overline{p}_t$ be an estimated transition dynamics. 
Let $V_h^{\pi}(s; l, p)$ be the value of a policy $\pi$ according to the cost and transition model $l, p$. 
Furthermore, let $\hat{V}_h^{\pi}(s; \overline{l}_t, \overline{p}_t)$ be a value function calculated by a truncated value estimation by the cost and transition model $\overline{l}_t, \overline{p}_t$. 
Assume the following holds for all $s, a, h, t \in [T]$:
\begin{enumerate}
    \item $\big|\hat{l}_h^t(s,a) - l_h(s,a)\big| \lesssim \sqrt{\frac{1}{N_h^t(s,a)}}$.
    \item $\big|\overline{p}_h^t(s' \mid s,a) - p_h(s' \mid s,a)\big| \lesssim \sqrt{\frac{p_h(s' \mid s,a)}{N_h^t(s,a) \lor 1}} + \frac{1}{N_h^t(s,a) \lor 1}$.
    \item $N_h^t(s,a) \leq \frac{1}{2}\sum_{j<t} q_h^{\pi_t}(s,a \mid p) - H \ln\frac{SAH}{\delta'}$.
\end{enumerate}
Furthermore, let $\pi_t$ be the policy by which the agent acts at the $t$-th episode. Then, for any $T' \in [T]$, we have:
\[
\sum_{t=1}^{T'} \left| V_1^{\pi_t}(s_1; l, p) - \hat{V}_1^{\pi_t}(s_1; \overline{l}_t, \overline{p}_t) \right| 
\leq \tilde{\mathcal{O}}\left(\sqrt{S \mathcal{N} H^4 T} + (\sqrt{\mathcal{N}} + H) H^2 S A \right).
\]
\end{lemma}

\begin{lemma}[Update rule recursion bound primal-dual \cite{efroni2020exploration}]\label{lem:update rule} 
\label{lem:primal-dual update}
    For any $\lambda \in \{ \lambda \in \mathbb{R}^m : 0 \leq \lambda \leq \rho \mathbf{1} \}$ and $T' \in [T]$,
\[
\sum_{t=1}^{T'} (-\tilde{y}_t^{\top} \lambda_t) + \sum_{t=1}^{N} \tilde{y}_t^{\top} \lambda \leq \frac{\eta_\lambda}{2} \|\lambda_1 - \lambda\|_2^2 + \frac{1}{2 \eta_\lambda} \sum_{t=1}^{T'} \|\tilde{y}_t\|^2.
\]
\end{lemma}

\begin{lemma}[\cite{efroni2020exploration}]\label{lem:F.2}
    Let $\lambda^{\star}$ be an optimal solution of the dual problem 
$$q_{\text{opt}} = \max_{\lambda \in \mathbb{R}^m_+, \mu \in \mathbb{R}^p} \left\{ q(\lambda, \mu) : (\lambda, \mu) \in \text{dom}(-q) \right\},$$
where $\text{dom}(-q) = \left\{(\lambda, \mu) \in \mathbb{R}^m_+ \times \mathbb{R}^p : q(\lambda, \mu) > -\infty \right\}$. And assume that $2\|\lambda^{\star}\|_1 \leq \rho$. Let $\tilde{\mathbf{x}}$ satisfy $\mathbf{A}\tilde{\mathbf{x}} + \mathbf{b} = 0$ and
\[
f(\tilde{\mathbf{x}}) - f_{\text{opt}} + \rho \|[y(\tilde{\mathbf{x}})]_+\|_\infty \leq \delta,
\]
\textit{then}
\[
\|[y(\tilde{\mathbf{x}})]_+\|_\infty \leq \frac{\delta}{\rho}.
\]
\end{lemma}

\begin{lemma}[\cite{efroni2020exploration}]\label{Lem:F.3}
    Let $\bar{x} \in X$ be a point satisfying $\mathbf{y}(\bar{x}) < 0$ and $\mathbf{A}\bar{x} + \mathbf{b} = 0$, and $\lambda^{\star}$ be an optimal dual solution. Then,
\[
\|\lambda^{\star}\|_1 \leq \frac{f(\bar{x}) - M}{\min_{j=1,\dots,m} -y_j(\bar{x})}.
\]
\end{lemma}



\end{document}